\documentclass[letterpaper, 10 pt, conference]{ieeeconf}  

\usepackage[switch]{lineno}

\IEEEoverridecommandlockouts 
\usepackage{amssymb, amsmath, amsfonts}

\usepackage{algorithm}
\usepackage{algorithmic}

\usepackage{graphicx}

\usepackage{booktabs}
\usepackage{multirow} 
\makeatletter

\newcommand{\Rmnum}[1]{\expandafter\@slowromancap\romannumeral #1@}
\makeatother

\usepackage{hyperref}

\newtheorem{assumption}{Assumption}

\title{\LARGE \bf
Safety Meets Speed: Accelerated Neural MPC with Safety Guarantees and No Retraining
}

\author{Kaikai Wang, Tianxun Li, Liang Xu, Qinglei Hu, Keyou You
\thanks{The work was supported in part by the National Natural Science Foundation of China under Grant 62373239, 62333011 and 62461160313, and the Project of Science and Technology Commission of Shanghai Municipality under Grant 22JC1401401.}
\thanks{Kaikai Wang, Liang Xu are with the School of Future Technology, Shanghai University, Shanghai, China (e-mail:\break
{\tt\small{\{\href{mailto:13102005639@shu.edu.cn}{13102005639}, \href{mailto:liang-xu@shu.edu.cn}{liang-xu}\}@shu.edu.cn}}).}
\thanks{Tianxun Li, Keyou You are with the Department of Automation, Tsinghua University, Beijing, China (e-mail:\break {\tt\small{\href{mailto:litx23@mails.tsinghua.edu.cn}{litx23@mails.tsinghua.edu.cn}}};
{\tt\small{\href{mailto:youky@tsinghua.edu.cn}{youky@tsinghua.edu.cn}}}).
}
\thanks{Qinglei Hu is with the School of Automation Science and Electrical Engineering, Beihang University, Beijing, China (e-mail:
{\tt\small{\href{mailto:huql_buaa@buaa.edu.cn}{huql\_buaa@buaa.edu.cn}}}).
}
}

\begin{document}


\maketitle
\thispagestyle{empty}
\pagestyle{empty}

\begin{abstract}
While Model Predictive Control (MPC) enforces safety via constraints, its real-time execution can exceed embedded compute budgets. We propose a Barrier-integrated Adaptive Neural Model Predictive Control (BAN-MPC) framework that synergizes neural networks' fast computation with MPC's constraint-handling capability. To ensure strict safety, we replace traditional Euclidean distance with Control Barrier Functions (CBFs) for collision avoidance. We integrate an offline-learned neural value function into the optimization objective of a Short-horizon MPC, substantially reducing online computational complexity. Additionally, we use a second neural network to learn the sensitivity of the value function to system parameters, and adaptively adjust the neural value function based on this neural sensitivity when model parameters change, eliminating the need for retraining and reducing offline computation costs. The hardware in-the-loop (HIL) experiments on Jetson Nano show that BAN-MPC solves 200 times faster than traditional MPC, enabling collision-free navigation with control error below 5\% under model parameter variations within 15\%, making it an effective embedded MPC alternative.

\end{abstract}
\begin{keywords}
Optimization and optimal control, machine learning for robot control, robust/adaptive control.

\end{keywords}

\section{INTRODUCTION}
Robots are widely deployed in search and rescue, obstacle avoidance navigation, and transportation scenarios due to their capability to replace or assist humans in complex tasks. Ensuring safety constitutes a fundamental requirement in robotic control for such applications. However, constrained working environments and onboard computational resources often compel systems to strike a balance between computational efficiency and safety assurance. Addressing these challenges necessitates developing safe, fast, and adaptive control strategies.

\subsection{Related Work}
Safe control in multi-obstacle or unstructured environments remains challenging. Traditional methods, such as the Dynamic Window Approach \cite{c5}, precomputed motion primitive libraries \cite{c7}, and collision-free flight corridors \cite{c9}, are ineffective for fast-moving robots or more complex environments. Reinforcement Learning (RL) \cite{y1} \cite{y2} achieves effective obstacle avoidance in robotic control by learning reward-optimizing actions from historical experiences. However, ensuring both safety and adaptability remains a core challenge in this domain. Model Predictive Control (MPC) ensures safety by integrating path planning, motion control, and constraint handling through constrained optimization. While applied to robotic obstacle avoidance in \cite{y3} \cite{y4}, most MPC methods employ rudimentary Euclidean norm-based safety criteria activated only near obstacles, potentially leading to insufficient safety margins or delayed reactions.

Recent studies focus on Control Barrier Functions (CBFs), which offer forward invariance and safety assurances, and have been applied across safe control domains including bipedal \cite{c10}, mobile \cite{a7}, and aerial robots \cite{c13}. Integrating CBFs with strict safety guarantees into MPC has created CBF-MPC \cite{c14}, a versatile tool for safe control, which inspired our work. However, as model complexity, system safety constraints, and online prediction steps increase, CBF-MPC may lead to significant computational time, which can affect the actual control performance.

Neural network-based Approximate MPC (AMPC) has become an effective solution \cite{c19}. Recent studies show that by avoiding online optimization, AMPC significantly reduces computational time, but typically loses the explicit constraint satisfaction guarantees. Orrico et al. \cite{a4} integrate neural value functions with Short-horizon MPC without addressing safety or adaptability. Xu et al. \cite{a10} penalize safety constraints in loss functions yet fail to ensure strict safety. Mitigation strategies --- Dijkstra-projected output constraints \cite{c22}, post-AMPC safety filters \cite{c23} --- compromise MPC performance. Consequently, accelerating the online computation of MPC while ensuring strict system safety remains an open problem.

Enhancing parameter adaptability of learning-based methods remains challenging. While techniques exist --- meta-learning \cite{c26} for knowledge transfer, domain adaptation \cite{c27} for reduced data dependency --- necessitate offline retraining with limited efficacy. Furthermore, adaptive approaches like warm-started online optimization \cite{c28} and RL-based network retraining \cite{c29} prove infeasible for fast-moving embedded robots demanding real-time control.

\subsection{Contributions}

This work investigates online path planning and motion control for obstacle avoidance, targeting safe, fast, and adaptive control in resource-constrained scenarios. We propose the BAN-MPC framework that enforces obstacle avoidance constraints via Control Barrier Functions integrated into MPC to ensure strict safety, while leveraging neural value functions and neural sensitivities to enhance computational efficiency and adaptability. HIL experiments on Jetson Nano demonstrate BAN-MPC's superior computational speed, safety, and adaptability over conventional MPC, establishing it as an effective embedded MPC alternative. The primary contributions are summarized below:

\begin{itemize}
\item Unlike computation-intensive CBF-MPC methods \cite{c14}, our framework embeds offline-learned neural value functions into Short-horizon MPC objectives, enabling rapid online solving of constrained optimization while preserving original state and input constraints.

\item Compared to neural network-based AMPC \cite{c19} \cite{a10} and RL \cite{y1} \cite{y2}, our approach provides strictly provable safety guarantees through CBF constraints. Innovatively, a second neural network learns neural sensitivities, enabling adaptive adjustments without network retraining when system parameters change, substantially reducing offline costs and enhancing policy generalization.

\item We deploy BAN-MPC on resource-constrained embedded platforms and validate its applicability through HIL experiments with unicycle and quadrotor.
\end{itemize}

The remainder of this paper is organized as follows: Section II presents the preliminary concepts. Section III introduces the BAN-MPC framework. Section IV provides the HIL experiments. Finally, Section V concludes the paper.

\section{PRELIMINARIES}

In this section, the fundamental concepts of CBF-MPC and parameter sensitivity are reviewed to establish a foundation for the proposed safe control framework presented later.
\subsection{CBF-MPC}

We consider a discrete-time nonlinear system:\begin{align}
x(t+1)=f_{\theta_{\text{nom}}}(x(t),u(t))
\end{align}
where $x(t)\in\mathbb{R}^{n_x}$ represents the system state, $u(t)\in\mathbb{R}^{n_u}$ is the control input at the current time $t$. The function $f:\mathbb{R}^{n_x}\times\mathbb{R}^{n_u}\to\mathbb{R}^{n_x}$ is a continuous dynamic function with nominal dynamics parameters $\theta_{\text{nom}}\in\Theta\subseteq\mathbb{R}^{q}$, such as the robot’s mass or the tire friction coefficient of an autonomous vehicle. Let the safety set for (1) be encoded as the superlevel set $\mathcal{S}$ of the function $H:\mathcal{X}\to\mathbb{R}$. 
\begin{align}
\mathcal{S}&=\{x\in\mathcal{X},H(x)\geq0\}\\
\partial \mathcal{S} &= \{x \in \mathcal{X}, H(x) = 0\}\\
\text{Int}(\mathcal{S}) &= \{x \in \mathcal{X}, H(x) > 0\}
\end{align}
Later, it will also be assumed that $\mathcal{S}$ has a nonempty interior and no isolated points, that is,
\begin{align}
\text{Int}(\mathcal{S}) \neq \emptyset \quad \text{and} \quad \overline{\text{Int}(\mathcal{S})} = \mathcal{S}.
\end{align}

\newtheorem{definition}{Definition}
\begin{definition}[Control Barrier Function \cite{a7}]    
The function $H(x):x\to\mathbb{R}$ is said to be a Control Barrier Function (CBF) if: $H(x_0) \geq 0$ and there exists a control input $u_k \in \mathcal{U}$ such that
\begin{align}
\Delta H(x_k,u_k)+\gamma H(x_k)\geq0, \quad 0 < \gamma \leq 1,
\end{align}
where $\Delta H(x_k,u_k)=H(x_{k+1})-H(x_{k})$, and $\gamma$ is the decay rate. 
\end{definition}

Based on this definition, we can establish a forward invariant safety condition as follows.

\newtheorem{lemma}{Lemma}
\begin{lemma}[Safety Condition \cite{a7}]    
Given a system (1), with a safe set $\mathcal{S}$ (2), and a CBF $H:\mathcal{S}\to\mathbb{R}$ defined by (6), for any ${x}_k\in\mathcal{X}$, any control input ${u}_k\in\mathcal{U}$, that satisfies
$\mathcal{U}_{cbf}=\{{u}_k\in\mathcal{U}: \Delta H({x}_k,{u}_k)+\gamma H({x}_k)\geq0\}$, renders the safe set $\mathcal{S}$ forward invariant for (1).
\end{lemma}

Integrating CBF as a constraint into the MPC framework, the resulting CBF-MPC \cite{c14} enables strict safety control through forward prediction, as expressed by the following:
\begin{subequations}
\begin{align}
V_{\text{MPC}}(x)&=\min_{\left\{x_{k}\right\}_{0}^{N},\left\{u_{k}\right\}_{0}^{N-1}}L\left(\{x_k\}_0^{N},\{u_k\}_0^{N-1}\right)\\
\mathrm{s.t.}\quad &x_{k+1}=f\left(x_{k},u_{k}\right)\quad\forall k\in[0,\dots,N-1], \\
&x_k\in\mathcal{X}_k\quad\forall k\in[0,\ldots,N], \\
&u_k\in\mathcal{U}_k\quad\forall k\in[0,\ldots,N-1], \\
&x_{0}=x,\\
\Delta H(&{x}_k,{u}_k)+\gamma H({x}_k)\geq0\ \quad \forall k \in [0, N-1].
\end{align}
\end{subequations}
where $k$ denotes the discrete-time index within the Optimal Control Problem (OCP), and $N$ represents the horizon length. The sequences $\{x_k\}_0^N$ and $\{u_k\}_0^{N-1}$ are the time series of state and input decision variables, while $\mathcal{X}_{k}$ and $\mathcal{U}_{k}$ denote the state and input constraints. $L\left(\{x_k\}_0^{N},\{u_k\}_0^{N-1}\right)$ represents the total cost, $V_{\text{MPC}}(x)$ is the optimal value function, and (7f) indicates the safety constraints imposed by the CBF.

\begin{figure*}
    \centering
     \vspace{-2mm}
    \includegraphics[width=1\linewidth]{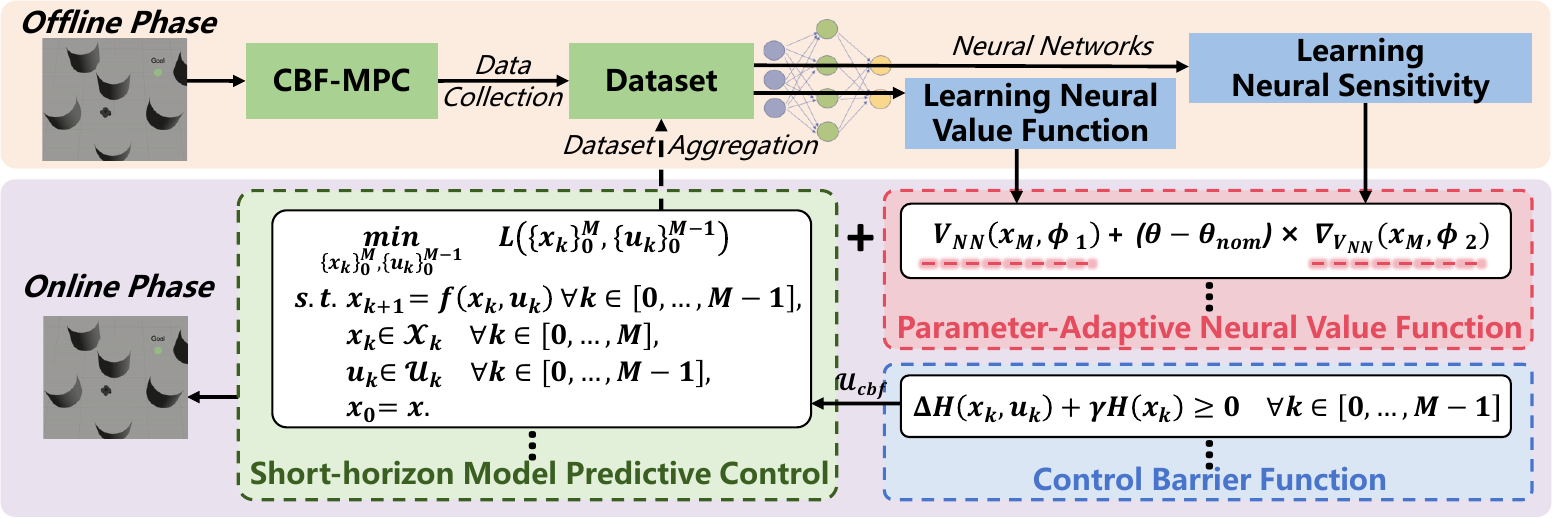}
    \caption{\textbf{Overview of BAN-MPC.} The framework is divided into an offline phase and an online phase. In the offline phase, training data is generated using CBF-MPC, and the neural value function and neural sensitivity are learned using neural networks. In the online phase, the neural value function and neural sensitivity are integrated into the optimization objective of Short-horizon MPC, with CBF ensuring strict safety guarantees.} 
    \vspace{-5mm}
    \label{fg:Figure_1}
\end{figure*}

\subsection{Parameter Sensitivity of MPC}

Sensitivity analysis measures the sensitivity of the optimal solution or value function to parameter changes. For simplicity, the MPC problem (7) is represented as a parameterized nonlinear programming (NLP) problem with parameter $\theta$:
\begin{subequations}
\begin{align}
V^{*}(\theta)=\min_{\mathbf{w}}J(\mathbf{w},\theta)\\
\mathrm{s.t.}\quad c(\mathbf{w},\theta)=0\\
g(\mathbf{w},\theta)\leq0
\end{align}
\end{subequations}
where $\mathbf{w}$ represents all decision variables, and the $\theta$ includes parameters such as the initial state and dynamic model parameters of (1). The Lagrangian function of (8) is given by:
\begin{align}
\mathcal{L}(\mathbf{w},\theta,\lambda,\mu):=J(\mathbf{w},\theta)+\lambda^\mathsf{T}c(\mathbf{w},\theta)+\mu^\mathsf{T}g(\mathbf{w},\theta)
\end{align}
where $\lambda$ and $\mu$ denote the Lagrange multipliers for constraints (8b) and (8c), respectively. The associated KKT conditions are:
\begin{subequations}
\begin{align}
\nabla_\mathbf{w} \mathcal{L}(\mathbf{w}, \theta, \lambda, \mu) &= 0\\
c(\mathbf{w}, \theta) &= 0  \\
g(\mathbf{w}, \theta) &\leq 0  \\
\mu_j g_j(\mathbf{w}, \theta) &= 0,\quad \mu_j \geq 0\quad \forall j 
\end{align}
\end{subequations}
 
For initial condition $\theta$, any point $\mathbf{s}^*(\theta) =[\mathbf{w}^*,\lambda^*,\mu^*]^\mathsf{T}$ satisfying these conditions constitutes a KKT point for $\theta$. Define the active constraint set $g_\mathbb{A}(\mathbf{w},\theta) \subseteq g(\mathbf{w},\theta)$ where $g_\mathbb{A}(\mathbf{w},\theta)=0$ at solution points. When $\mu_\mathbb{A}>0$ (indicating no weakly active constraints), strict complementarity holds, reducing the KKT system to:
\begin{align}
 \varphi(\mathbf{s}(\theta),\theta)=\begin{bmatrix}\nabla_\mathbf{w}\mathcal{L}(\mathbf{w},\theta,\lambda,\mu)\\c(\mathbf{w},\theta)\\g_\mathbb{A}(\mathbf{w},\theta)\end{bmatrix}=0
\end{align}

In performing parameter sensitivity analysis and derivation, model parameter variations need to be systematically handled. It is assumed that the cost and constraints are twice continuously differentiable in the vicinity of the KKT point, satisfying necessary conditions.

\begin{assumption}
The cost and constraints in equation (8) are twice continuously differentiable in the neighborhood of the KKT point $\mathbf{s}^*(\theta)$. The linear independence constraint qualification (LICQ), the second-order sufficient conditions (SOSC), and the strict complementarity of the solution vector $\mathbf{s}^{*}(\theta)$ hold.
\end{assumption}

\begin{assumption}
For all parameter perturbations $\Delta\theta$ in a neighborhood of 
$\theta$, the set of active constraints remains unchanged; that is, the KKT points $s^*(\theta)$ and $s^*(\theta+\Delta\theta)$ share the same active set $g_\mathbb{A}$.
\end{assumption}

Given Assumption 1 and 2, for a parameter perturbation $\Delta \theta$ around $\theta$, there exists a unique continuous vector function $\mathbf{s}^*(\theta+\Delta \theta)$, which is the KKT point of the problem with parameter $\theta+\Delta \theta$. By applying the implicit function theorem to the gradients of the KKT conditions, a first-order prediction of the feasible solution can be obtained \cite{c30}.
\begin{align}
    \mathbf{s}^*(\theta+\Delta \theta)\approx\mathbf{s}^*(\theta)+\frac{\partial\mathbf{s}^*}{\partial \theta}\Delta \theta
\end{align}
where the gradient $\frac{\partial\mathbf{s}^*}{\partial \theta}$ represents the sensitivity to parameter changes. Using the predicted solution $\mathbf{s}^*(\theta+\Delta \theta)$, the corresponding optimal value function ${V}^*(\theta+\Delta \theta)$ can be derived.

\section{BARRIER-INTEGRATED ADAPTIVE NEURAL MODEL PREDICTIVE CONTROL}

In this section, we present the overall framework of BAN-MPC and provide a detailed explanation of the learning methods for the neural value function and its sensitivity. Fig. \ref{fg:Figure_1} provides an overview of BAN-MPC, which consists of two stages: the offline phase and the online phase.

To accelerate the online computation of CBF-MPC, we integrate neural networks into our framework. In the offline phase, instead of using neural networks to approximate the entire MPC policy as in AMPC, our learning objective is to offline learn the cost function. Using the CBF-MPC framework in OCP (7), we collect training data and employ the neural network $V_{NN}(x)$ to approximate the optimal value function $V_{\text{MPC}}(x)$ in OCP (7). To improve the adaptability of the control strategy, we extend the adaptive method in \cite{a1}. A second neural network, $\nabla_{V_{NN}}(x)$, is used to offline learn the sensitivity of the value function to parameters. When model parameters change, the neural sensitivity $\nabla_{V_{NN}}(x)$ adaptively adjusts $V_{NN}(x)$ without retraining the value function.

In the online phase, we propose shortening the online prediction horizon of the CBF-MPC problem (7) ($M\ll N$), transforming it into a short-sighted model predictive controller to minimize online computation time. The prediction performance for the remaining $N-M$ steps will be represented by the the parameter-adaptive neural value function $V_{\text{BAN-MPC}}(x,\theta)=V_{\text{NN}}(x)+(\theta-\theta_{\text{nom}})\times\nabla_{V_{\text{NN}}}(x)$. The proposed BAN-MPC problem can be formulated as follows:
\begin{subequations}
\begin{align}
\min_{\left\{x_{k}\right\}_{0}^{M},\left\{u_{k}\right\}_{0}^{M-1}}&L\left(\{x_k\}_0^{M},\{u_k\}_0^{M-1}\right)+V_{\text{BAN-MPC}}(x_M,\theta)\\
\mathrm{s.t.}\quad x_{k+1}=&\;f\left(x_{k},u_{k}\right)\quad\forall k\in[0,\dots,M-1], \\
x_k\in &\;\mathcal{X}_k\quad\forall k\in[0,\ldots,M], \\
u_k\in &\;\mathcal{U}_k\quad\forall k\in[0,\ldots,M-1], \\
x_{0}= &\;x,\\
\Delta H({x}_k,{u}_k)&+\gamma H({x}_k)\geq0\ \quad \forall k \in [0, M-1].
\end{align}
\end{subequations}
where (13a) represents the optimization objective incorporating the neural value function and neural sensitivity, while (13f) defines the obstacle avoidance constraints using CBF. While Assumption 1 requires second-order differentiable CBF constraints, obstacle avoidance tasks typically involve $Q$ non-smooth safety functions $H^{(q)}$. We adopt the log-sum-exp soft minimum operation to smoothly compose multiple safety functions,
\begin{subequations}
\begin{align}
H(x)&= \operatorname{softmin}_{\rho}(H^{(1)}(x),H^{(2)}(x),\ldots,H^{(Q)}(x))\\
&=-\frac{1}{\rho}\,\log\!\Bigl(\,\sum_{q=1}^{Q} e^{-\rho H^{(q)}(x)}\Bigr)
\end{align}
\end{subequations}
where $\rho>0$, and the $C^{\infty}$ property of log-sum-exp ensures that both the composite $H(x)$ and the resulting CBF constraint (13f) are twice continuously differentiable. For stricter constraints, one can adopt the exact dual-cone formulation \cite{x1} to eliminate the approximation entirely. Although this study focuses on static obstacles, adding obstacle-state sensing and trajectory prediction enables the safety function $H^{(q)}(x)$ to be reformulated as $H^{(q)}(x,t)$; embedding time in the CBF constraint $\Delta H({x}_k,{u}_k,t_k)+\gamma H({x}_k,t_k)\geq0$ extends the method to moving obstacles.

\subsection{Efficient Learning of Neural Value Function} 

Short-horizon MPC, which focuses on short-term costs within the predictive horizon, significantly reduces the complexity of the CBF-MPC Problem (7) and accelerates computation, but its short-sighted nature results in substantial degradation in control performance. Therefore, we propose leveraging neural networks to learn the MPC value function (termed the neural value function) and integrate it into the Short-horizon MPC framework, achieving faster computation without sacrificing control performance. As the model complexity increases, the training data often fails to cover the data distribution encountered in real-world deployment. Traditional behavioral cloning methods may exhibit degraded performance due to distribution shifts, causing the learned model to deviate from the optimal value function. 

\vspace{-2mm}
\begin{algorithm}[!ht]
\setlength{\itemsep}{1pt} 
\caption{VF-DAGGER algorithm}
\label{algo:ratio}
\begin{algorithmic}[]
\STATE \textbf{Require:} expert MPC policy $\pi^*$, iterations number $n_{D}$ 
\STATE \textbf{1)} Generate dataset $\mathcal{D}$ by policy $\pi^*$
\STATE \textbf{2)} Train the value function $V^*$ on dataset $\mathcal{D}$ using (15)
\STATE \textbf{3)} Let $V_1=V^*$
\STATE \textbf{4)} \textbf{for} $i = 1$ \textbf{to} $n_{D}$ \textbf{do}
    \STATE \textbf{5)} \hspace{1em} \parbox[t]{0.8\linewidth} {Obtain $\pi_i$ by embedding $V_i$ into the value function of (13)}\vspace{1pt}
    \STATE \textbf{6)} \hspace{1em}  Sample T-step trajectories using $\pi_i$
    \STATE \textbf{7)} \hspace{1em} \parbox[t]{0.8\linewidth} {Generate dataset $\mathcal{D}_i$ of visited states by $\pi_i$ and actions given by expert}\vspace{1pt}
    \STATE \textbf{8)} \hspace{1em}  Aggregate dataset: $\mathcal{D} \leftarrow {\mathcal{D}} \cup {\mathcal{D}_i}$
    \STATE \textbf{9)} \hspace{1em} \parbox[t]{0.8\linewidth} {Train the value function $\hat V_{i+1}$ on dataset $\mathcal{D}$ using (15)}\vspace{1pt}
    \STATE \textbf{10)} \hspace{1em} Let $V_{i+1}=\beta_i V^*+(1-\beta_i)\hat V_{i+1}$
\STATE \textbf{11)} \textbf{end for}
\STATE \textbf{12)} Return best $V_{\text{NN}}=V_i$ on validation
\end{algorithmic}
\end{algorithm}

We propose Value Function Dataset Aggregation (VF-DAGGER), which inherits the training philosophy of DAGGER \cite{a3}. It mitigates distribution shift by iteratively merging initial demonstrations with actual policy state distributions into hybrid datasets (Algorithm \ref{algo:ratio}). At each iteration, datasets generated by the CBF-MPC policy with embedded current neural value functions are aggregated into the original dataset. The value function is then learned on the augmented dataset and iteratively integrated with the initial neural value function, ensuring that the learned neural value function remains close to the value function of the original CBF-MPC problem (7). Specifically, the original training dataset is collected by offline solving the CBF-MPC (expert MPC) problem (7) with long prediction horizons under appropriate sampling strategies. The resulting optimal state-action-value dataset, denoted as 
$\mathcal{D}:=\{(x_{i},u_{i}^{*},V_{\text{MPC}}(x_{i}))\}_{i=1}^{n_{\mathrm{tr}}}$, where $n_{\mathrm{tr}}$ is the number of training samples, is utilized to learn the initial neural value function $V^*$, where the mean squared error (MSE) is adopted as the loss function:
\begin{align}
    \mathcal{L}_{\text{MSE}}(\mathcal{D}, \phi)
=\frac{1}{n_{tr}} \sum_{i=1}^{n_{tr}} |V(x_i; \phi) - V_{MPC}(x_i)|^2
\end{align}
where $\phi$ denotes the parameters learned by the neural network. During each iteration, the current neural value function $V_i$ is embedded into the Short-horizon MPC to formulate the BAN-MPC policy $\pi_i$. The policy $\pi_i$ is used to sample T-step trajectories to generate state sequences, while the expert policy $\pi^*$ provides the corresponding actions, forming a new dataset $\mathcal{D}_i$. This dataset is then aggregated with the original dataset to form an augmented dataset $\mathcal{D}$. The neural value function $\hat V_{i+1}$ is subsequently retrained on this augmented dataset $\mathcal{D}$ while employing the same loss function as specified in equation (15). To better approximate the value function of expert MPC, we propose modifying the neural value function $V_{i+1}$:
\begin{align}
    V_{i+1}=\beta_i V^*+(1-\beta_i)\hat V_{i+1}
\end{align}
where $\beta_i$ denotes an optional weighting coefficient. This design is often recommended in practice, as the learned value function may deviate from the expert MPC’s during initial iterations due to insufficient training data. After sufficient iterations, the best-performing model on the validation set is selected as $V_{\text{NN}}$, approximating $V_{\text{MPC}}$ of CBF-MPC.

\subsection{Learning Neural Sensitivity for Parameter-Adaptive Neural Value Function}
When model parameters change, the original neural network-based strategy may fail, requiring retraining and incurring significant offline computational costs. A direct approach that concatenates model parameters $\theta \in \mathbb{R}^{q}$ with system states $x \in \mathbb{R}^{n_x}$ as network inputs leads to a combinatorial explosion in the required training data. With $K$ samples required per input dimension, the training set grows as $\mathcal{O}(K^{n_x + q})$.

In contrast, we leverage the parameter sensitivity of the MPC problem to predict how the optimal value function responds to parameter changes online, avoiding the need for retraining. We select the nominal dynamics parameters $\theta_{\text{nom}}$ and learn the neural value function following the similar learning methods described in Section III-A. The value function of the CBF-MPC as expressed in (7) is denoted as $V_{\text{MPC}}(x,\theta_{\text{nom}})$, while the learned neural value function is represented as $V_{\text{NN}}(x)$. Additionally, we use a second neural network to learn the sensitivity of the value function with respect to parameters $(x_i,\left.\frac{\partial}{\partial\theta}V_{\text{MPC}}(x_i,\theta)\right|_{\theta_{\text{nom}}})$, with the corresponding training dataset denoted as $\hat{\mathcal{D}}\colon=\{(x_{i},\frac{\partial}{\partial\theta}V_{\text{MPC}}(x_{i},\theta))\}_{i=1}^{n_{\mathrm{tr}}}$. We apply the DAGGER method to learn the neural sensitivity, represented by $\nabla_{V_{\text{NN}}}(x):\mathbb{R}^n\to\mathbb{R}^{m\times p}$, which outputs an approximation of the partial derivative (Jacobian) $\frac{\partial}{\partial\theta}V_{\text{MPC}}$ evaluated at $\theta_{\text{nom}}$. Using the two neural networks, $V_{\text{NN}}(x)$ and $\nabla_{V_{\text{NN}}}(x)$, we can provide a parameter-adaptive neural value function:
\begin{align}
    V_{\text{BAN-MPC}}(x,\theta)=V_{\text{NN}}(x)+\nabla_{V_{\text{NN}}}(x)(\theta-\theta_{\text{nom}})
\end{align}

In light of (12) and its validity conditions, the real system parameters $\theta$ must stay close to the nominal values $\theta_{\text{nom}}$, at least such that the active constraint set remains unchanged as stipulated in Assumption 2. The parameter-adaptive neural value function uses the nominal parameters to approximate the optimal value function, and the prediction factor is given by the approximate sensitivity $\nabla_{V_{\text{NN}}}(x)$. When $\theta$ varies, the approximate optimal value function can be efficiently computed by evaluating $V_{\text{NN}}$ and $\nabla_{V_{\text{NN}}}(x)$ once at $\theta_{\text{nom}}$, followed by a simple vector multiplication and addition. This avoids a full forward pass with $(x,\theta)$ as input and significantly reduces both the offline dataset dimensionality and the online computational burden. The Probabilistic Practical Exponential Stability of BAN-MPC is established, with detailed proofs and conclusions provided in the Appendix.

\section{HIL EXPERIMENTS}

In this section, we conduct Hardware-in-the-Loop (HIL) experiments on both a unicycle and a quadrotor to validate the efficacy of the proposed control framework for safe control on embedded devices, as shown in Fig. \ref{fg:Figure_2}.

\begin{figure}[!ht]
\centering
    \includegraphics[width=0.6\linewidth]{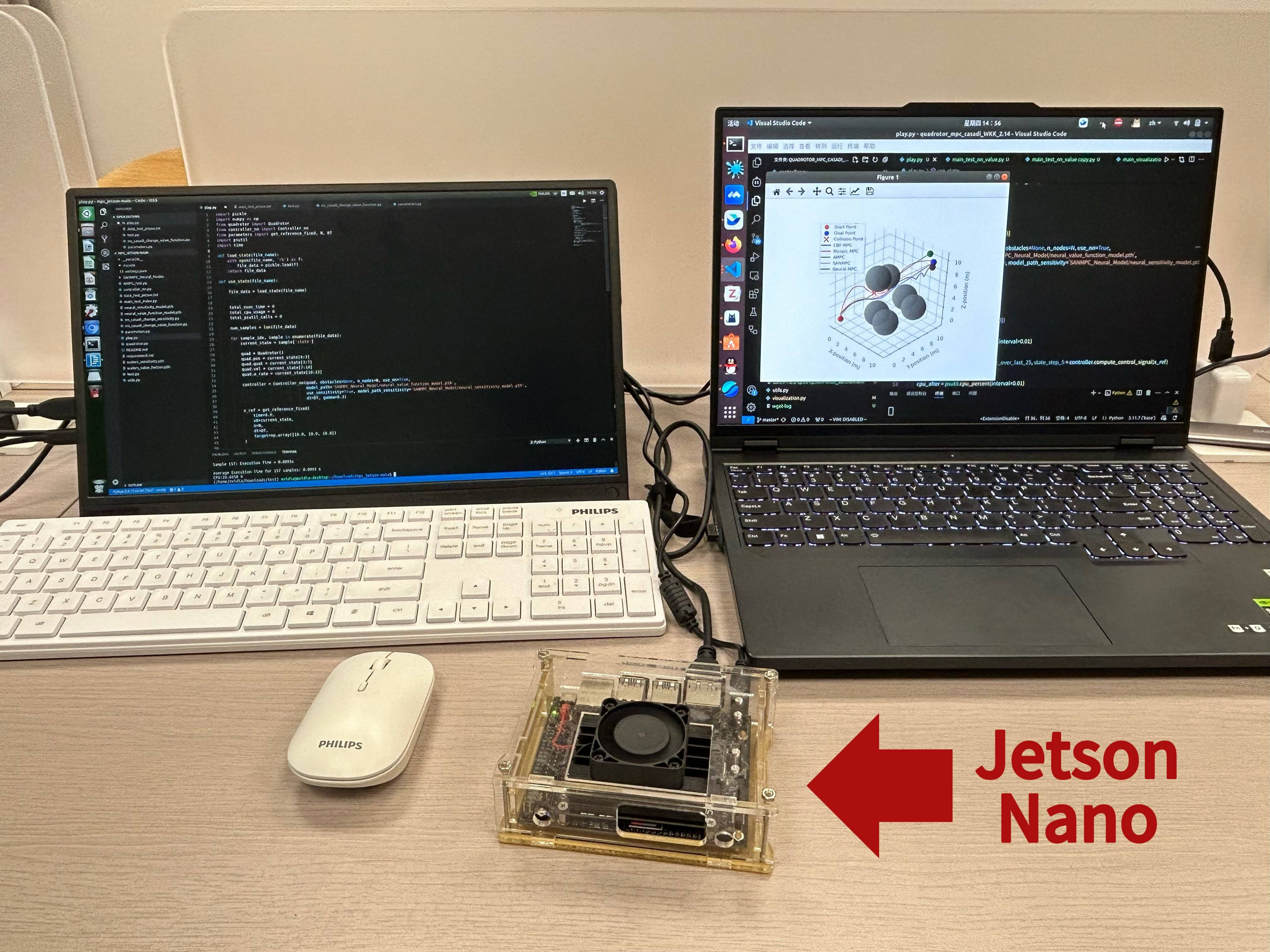}

    \caption{HIL evaluation system based on Jetson Nano.} 
    \label{fg:Figure_2}
\vspace{-5mm}
\end{figure}

\subsection{Baseline Controllers}
To evaluate the quality and computational improvements of the proposed method, we compared the following controller designs: 
\begin{itemize}
\item \textbf{CBF-MPC} is a controller based on OCP (7), where we use the quadratic cost function:

$L\left(\{x_{k}\}_{0}^{N},\{u_{k}\}_{0}^{N-1}\right)=\sum_{k=0}^{N}\|x_{k}\|_{Q}^{2}+\sum_{k=0}^{N-1}\|u_{k}\|_{R}^{2}$.

with a prediction horizon of $N=30$. The CBF adopts the standard form of (7f), where $H(x)$ represents the distance from the controlled object to the obstacle boundary, and $\gamma =0.3$.

\item \textbf{Short-horizon MPC} shortens the prediction horizon of CBF-MPC to $N=3$.

\item \textbf{AMPC} employs a neural network to approximate the CBF-MPC problem (7), similar to the approach in \cite{c19}. The MPC mapping is approximated from the current sampling point $x(t_{i})$ to the first optimal input $u_0^*(t_i)$.

\item \textbf{Neural MPC} integrates the neural value function learned in Section III-A into the optimization objective of Short-horizon MPC.

\item \textbf{BAN-MPC} extends Neural MPC with neural sensitivity for adaptive control. When model parameters are fixed, BAN-MPC aligns with Neural MPC.
\end{itemize}

\subsection{Simulation Setting}
Jetson Nano leverages its parallel processing capabilities to significantly accelerate the evaluation of neural networks. In Fig. \ref{fg:Figure_3}, we introduce a HIL evaluation system to facilitate the implementation of BAN-MPC and comparative experiments on the Jetson Nano, which comprises a host personal computer (PC) and a Jetson Nano embedded device. The host PC converts offline-trained PyTorch neural network models (i.e., neural value function and neural sensitivity) into embeddable CasADi \cite{a6} symbolic expressions for value function integration, subsequently deploying neural network models, dynamics models, and test data to the Jetson platform. Preconfigured with Ubuntu 18.04, Torch 1.10, and CasADi 3.5, the Jetson Nano executes real-time BAN-MPC computations, with test results returned to the PC for data processing and visualization. Secure SSH-based remote communication and SFTP-enabled Xftp tools ensure reliable bidirectional data transfer between both systems.

\begin{figure}[ht]
\centering
    \includegraphics[width=1\linewidth]{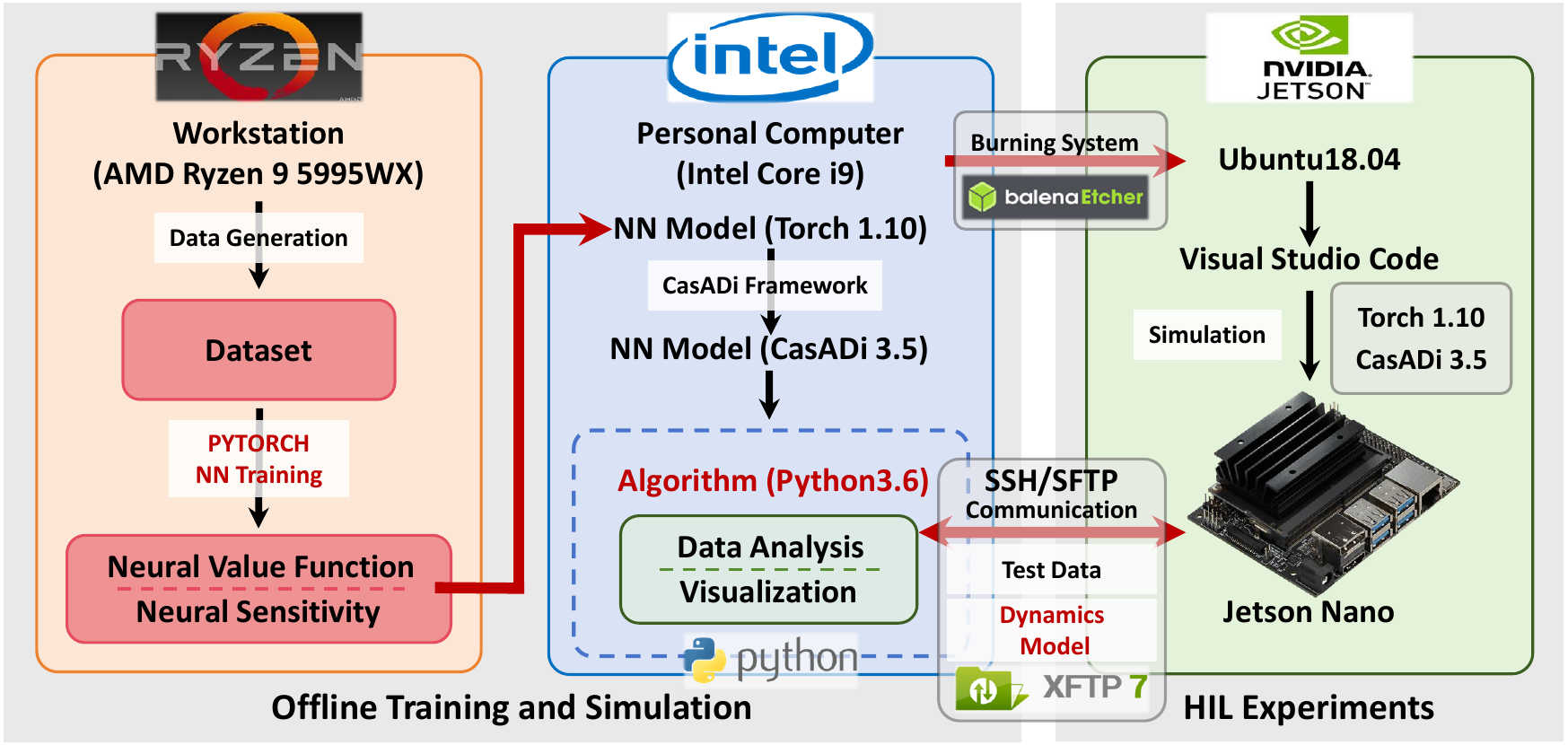}
    \vspace{-5mm}
    \caption{Hardware-in-the-Loop Evaluation System Diagram.} 
    \vspace{-2mm}
    \label{fg:Figure_3}
\end{figure}
In our experiments, all learning-based controllers are implemented using neural networks with three hidden layers, each comprising 32 neurons and utilizing the tanh activation function. IPOPT is used as the nonlinear solver for solving the optimization problems involved in these controllers. For data generation, we generated $5 \times 10^{5}$ samples offline for the unicycle model and $10^{7}$ samples offline for the quadrotor model. The data generation process was carried out on a workstation with 128 cores and an AMD Ryzen 9 5995WX.

\subsection{Unicycle for Obstacle Avoidance and Navigation}

In the first example, we consider the unicycle system \cite{a5}:
\begin{align}
\begin{bmatrix}
\dot{x} \\
\dot{y} \\
\dot{\psi}
\end{bmatrix}=A(\psi)u,A(\psi)=
\begin{bmatrix}
\cos(\psi) & 0 \\
\sin(\psi) & 0 \\
0 & 1
\end{bmatrix},u=
\begin{bmatrix}
v \\
\omega
\end{bmatrix}
\end{align}
where $(x,y)$ represents the position, $\psi$ denotes the heading angle, $v$ is the linear velocity, and $\omega$ is the angular velocity. The input constraints are $\omega\in[-1.8,1.8]$ and $v\in[-0.26,0.26]$. The task is to navigate a unicycle with a radius of 0.1 from the start to the goal while avoiding five circular obstacles of varying sizes. A distance of less than 0.03 from the obstacle's edge is considered unsafe. 

\begin{figure}[h!]
\centering
    \includegraphics[width=0.7\linewidth]{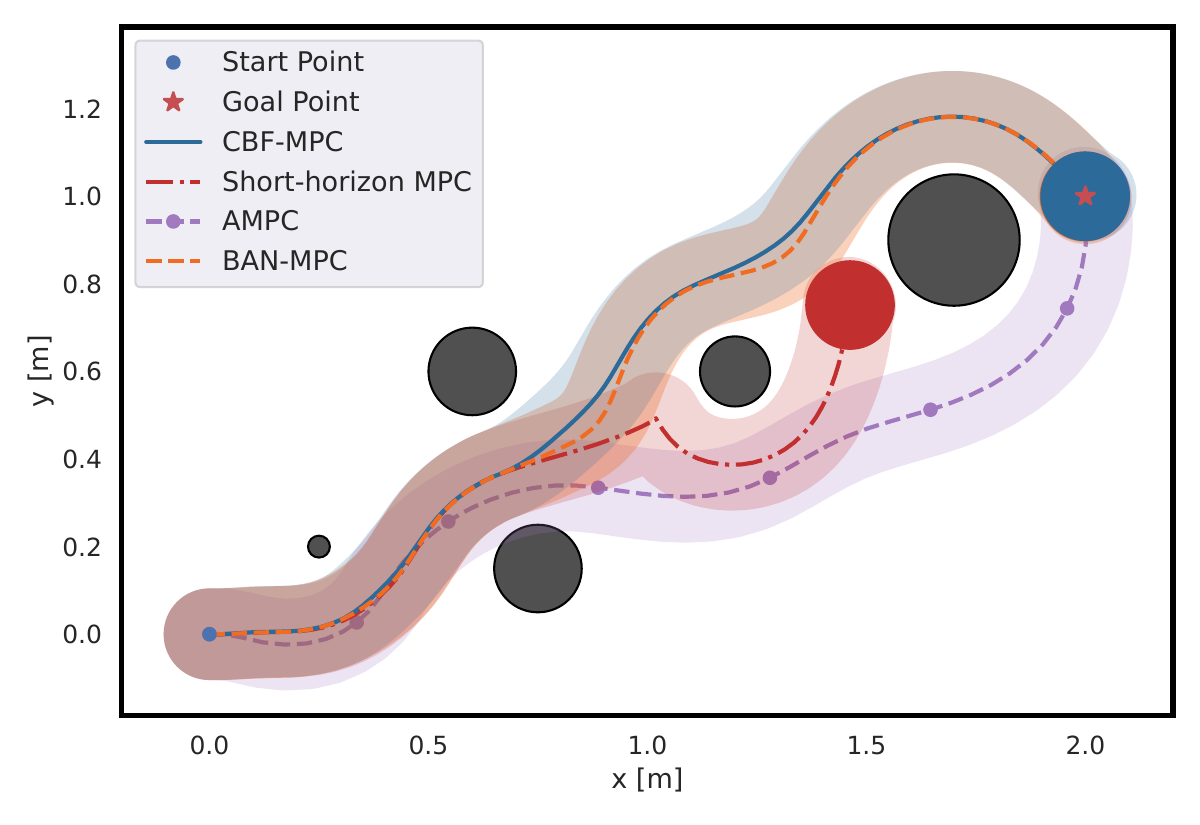}
\vspace{-2mm}
    \caption{Obstacle avoidance navigation task for the unicycle. The gray circles represent obstacles, the blue dot marks the starting point, and the red star denotes the target point.} 
    \vspace{-2mm}
    \label{fg:Figure_4}
\end{figure}

\begin{figure}[!h]
\centering
    \includegraphics[width=0.7\linewidth]
    {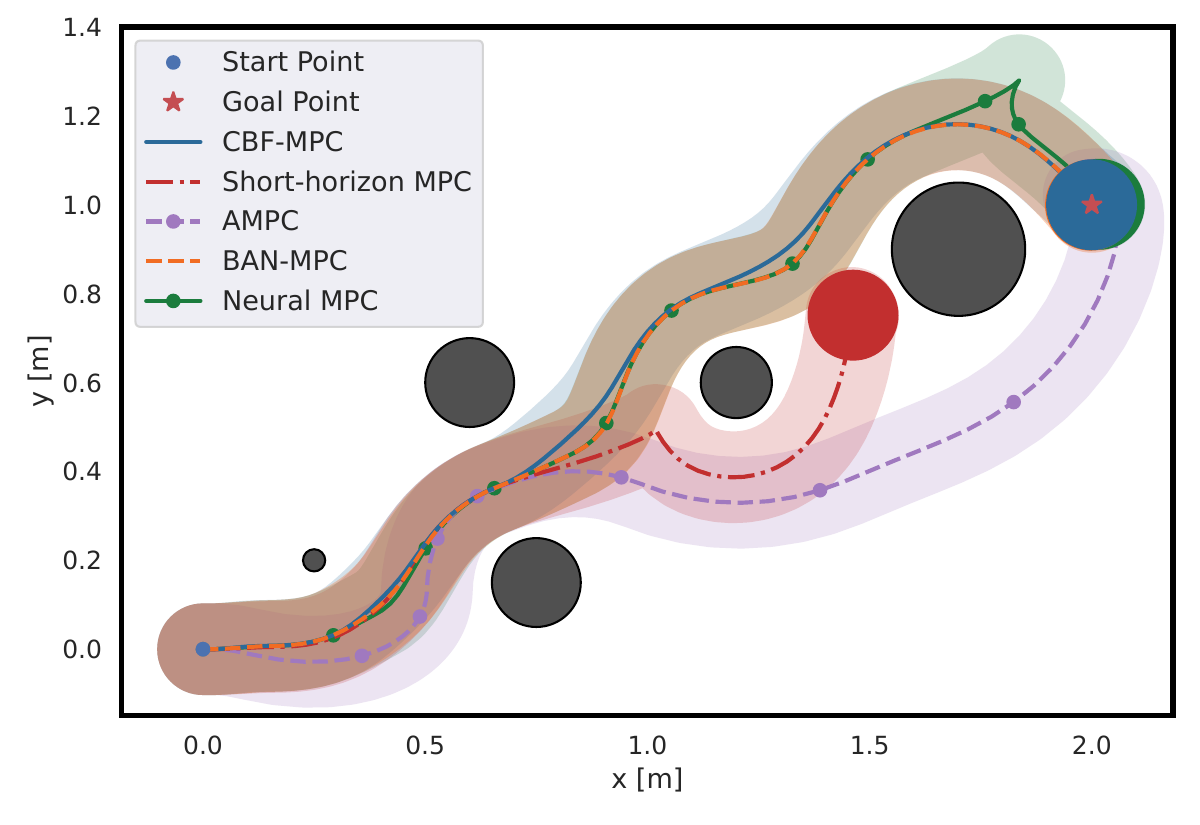}
\vspace{-2mm}
    \caption{Obstacle avoidance navigation task for the unicycle with a changed system model. The start point, goal point, and obstacles remain unchanged.} 
    \label{fg:Figure_5}
    \vspace{-2mm}
\end{figure}

\begin{table}[!ht]
\centering
\caption{CONTROL METRICS FOR UNICYCLE}
\label{table1}
\begin{tabular}{cccc}
\toprule
Baseline & Domain & Average & Computation \\ controllers &  safety(\%) & suboptimality(\%) & time (s) \\
\midrule
CBF-MPC & 100 & 0 & 0.4294\\
AMPC & 92.40 & 2.14 & 0.0080\\
BAN-MPC & 100 & 0.26 & 0.0627\\
\bottomrule
\end{tabular}
\vspace{-2mm}
\end{table}

Fig. \ref{fg:Figure_4} compares the performance of several controllers tested on the Jetson Nano (BAN-MPC is equivalent to Neural MPC when parameters are unchanged). Short-horizon MPC fails due to short-sightedness, while AMPC lacks strict safety guarantees and relies on neural approximations, leading to unsafe behavior and trajectory deviations. In contrast, BAN-MPC exhibits superior control performance, with state trajectories nearly identical to those of CBF-MPC. Furthermore, Table \ref{table1} provides detailed comparison metrics. Domain safety evaluates 10,000 test points, quantifying the proportion of states within the safe set. Average suboptimality reflects the average relative cost difference from CBF-MPC, and computation time is the average time to compute control inputs per instance. Notably, as Short-horizon MPC failed to reach the target state, no data is recorded. BAN-MPC exhibits superior control performance, safety, and computational efficiency over traditional MPC for unicycle obstacle avoidance.

To evaluate the adaptability of BAN-MPC, we modify matrix $A$ to simulate changes in the system model, such as variations in friction or mass. Fig. \ref{fg:Figure_5} shows that Neural MPC suffers significant deviations during sharp turns due to error accumulation, while BAN-MPC adaptively adjusts inputs using neural sensitivity, keeping the state closely aligned with CBF-MPC and demonstrating excellent adaptability.

\subsection{Quadrotor for Obstacle Avoidance and Navigation}

\begin{figure*}[ht]
\centering
    \includegraphics[width=0.85\linewidth]{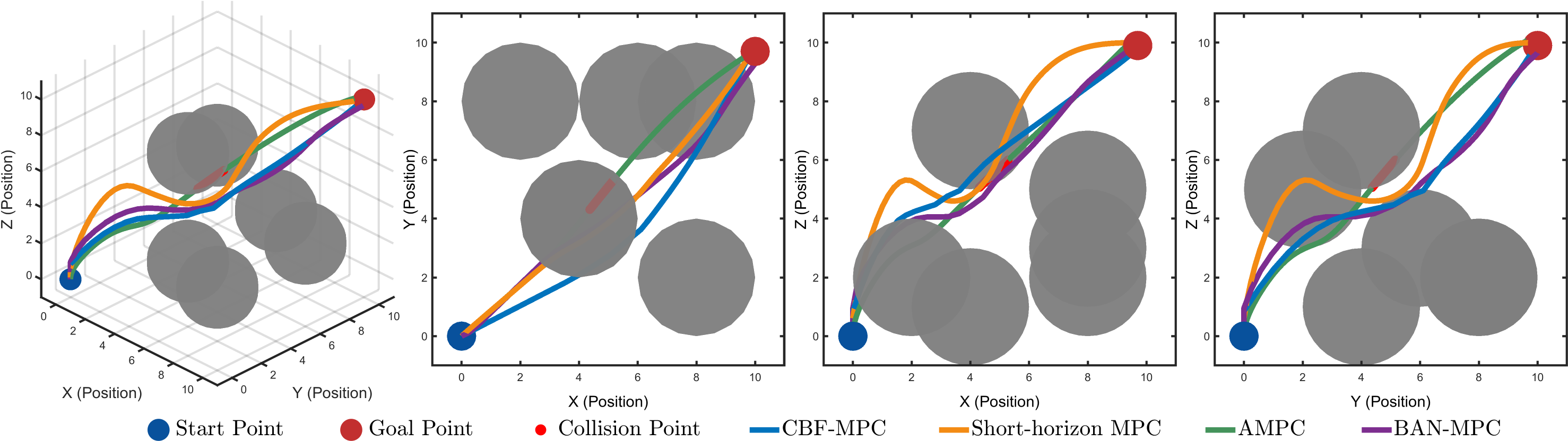}
    \caption{Quadcopter obstacle avoidance navigation trajectory and projection plot.} 
    \label{fg:Figure_6}
\end{figure*}

\begin{figure*}[ht]
\centering
    \includegraphics[width=0.85\linewidth]{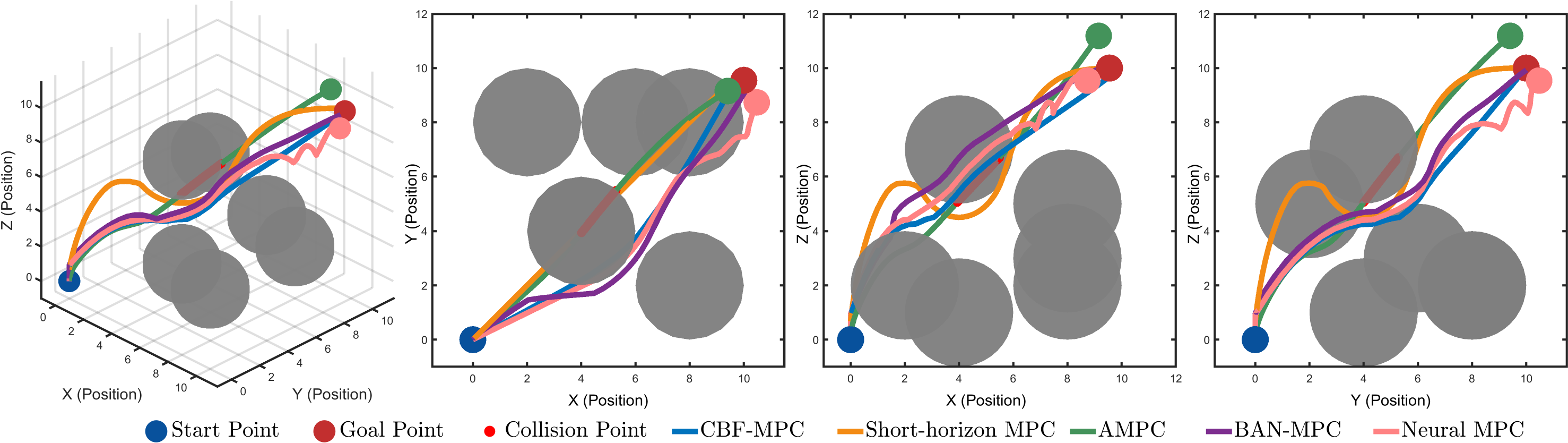}
    \caption{Quadcopter obstacle avoidance navigation trajectory and projection plot after modifying mass parameters.} 
    \label{fg:Figure_7}
  
\end{figure*}

\begin{table*}[ht]
\centering
\caption{COMPARISON OF DIFFERENT CONTROLLERS ACROSS THE PERSONAL COMPUTER AND JETSON NANO}
\label{table2}
\setlength{\tabcolsep}{0.4cm} 
\begin{tabular}{c|cccc|ccc}
\toprule
Platform & \multicolumn{4}{c|}{Intel Core i9@2.2GHz} &\multicolumn{3}{c}{Jetson Nano@1.42GHz} \\
\midrule
Baseline controlers & CBF-MPC & Short-horizon MPC & AMPC & BAN-MPC & CBF-MPC &  AMPC & BAN-MPC \\
\midrule
Domain safety(\%) & 100 & 100 & 97.13 & 100 & 100 & 96.98 & 100 \\
\midrule
Boundary safety(\%) & 100 & 100 & 49.55 & 100 & 100 & 50.05 & 100 \\
\midrule
Computation time(s) & 1.6029 & 0.0525 & 0.0010 & 0.0804 & 2.2008 & 0.0008 & 0.0096 \\
\midrule
Utilization(\%) & 0.56 & 0.98 & 5.57 & 0.69 & 2.24 & 13.24 & 25.95 \\
\bottomrule
\end{tabular}
\vspace{-3mm}
\end{table*}

Next, we adopt the 6-degree-of-freedom rigid body quadrotor model from \cite{a8}. The diagonal moment of inertia matrix is given by $J=diag(J_x,J_y,J_z)$, and the control input $u$ consists of the rotor thrusts $T_i, \forall i \in \{0, 1, 2, 3\}$. The position of the quadrotor in the world coordinate system is denoted as $p_{WB}$, its attitude is represented by $q_{WB}$, its linear velocity is given by $v_{WB}$, and the angular velocity is expressed as $\omega_B$. The thrusts on each rotor are modeled separately, resulting in the total thrust $T_B$ and the body torque $\tau_B$:
\begin{align}
T_B = \begin{bmatrix} 
0 \\
0 \\
\sum T_i 
\end{bmatrix}
,
\tau_B = \begin{bmatrix}
d_y (-T_0 - T_1 + T_2 + T_3) \\
d_x (-T_0 + T_1 - T_2 + T_3) \\
c_\tau (-T_0 + T_1 - T_2 - T_3)
\end{bmatrix}
\end{align}
where, $d_x$ and $d_y$ represent the displacements from the rotors to the quadrotor's center of mass, and $c_\tau$ is the rotor drag moment constant. The 13-dimensional nonlinear state-space model of the system is described by the following equation:
\begin{align}
\dot{x} = f(x, u) =
\begin{bmatrix}
v_W \\
q_{WB} \cdot \left[ \begin{matrix} 0 \\ \omega_B / 2 \end{matrix} \right] \\
q_{WB} \odot T_B + g_W \\
J^{-1} (\tau_B - \omega_B \times J \omega_B)
\end{bmatrix}
\end{align}
where $x = \left[ p_{WB}, q_{WB}, v_{WB}, \omega_B \right]^T$ is the state vector and $g_W = \left[ 0, 0, -9.81 \, \mathrm{ms}^{-2} \right]^T$ denotes the Earth’s gravity. To incorporate these dynamics into discrete-time algorithms, we utilize the 4th-order explicit Runge-Kutta method, denoted as $f_{RK4}(x,u)$, to integrate $\dot{x}$ with respect to the initial state $x_k$, input $u_k$, and integration step $\delta t$, resulting in the discrete dynamics:
\begin{align}
x_{k+1} = f_{RK4}\bigl(x_k, u_k, \delta t\bigr)
\end{align}

The task is to navigate the quadrotor, modeled as a 0.5m radius sphere, from start to goal while avoiding six spherical obstacles. Fig. \ref{fg:Figure_6} illustrates the trajectory profiles of quadrotor navigation under different control strategies, with corresponding two-dimensional projections onto the xy-, xz-, and yz-planes (BAN-MPC is equivalent to Neural MPC when parameters are unchanged). Short-horizon MPC, due to its short-sighted control, deviates significantly from CBF-MPC's trajectory and performs poorly. Meanwhile, AMPC fails to ensure safety at unobserved points, leading to a collision. In contrast, BAN-MPC closely follows the baseline CBF-MPC trajectory, demonstrating superior performance.

To evaluate the adaptability of BAN-MPC, we modified the quadrotor's mass parameter from 1kg to 1.2kg. The results are shown in Fig. \ref{fg:Figure_7}. When using Neural MPC as the controller, the quadrotor trajectory exhibits noticeable deviations and fluctuations due to accumulated errors. In contrast, BAN-MPC, with its parameter adaptation capability, maintains performance close to the baseline CBF-MPC.

We provide detailed metrics in Table \ref{table2} (Short-horizon MPC is computationally infeasible on the Jetson Nano due to numerical instability under complex constraints, causing solve times to exceed hardware limits, so no results are recorded; While feasible on the PC, it failed to converge within real-time constraints on the embedded platform in $>90\%$ of trials). Domain safety reflects the overall safety of the quadrotor in the state space, quantified by randomly sampling 10,000 points and calculating the proportion remaining within the safe set. Boundary safety reflects obstacle avoidance near the safety boundary, evaluated by sampling 10,000 points within 0.5m of the obstacles and calculating the proportion within the safe set. Domain safety focuses on overall safety, while boundary safety addresses safety near critical boundaries or specific boundaries. Computation time refers to the average time taken to solve the control inputs in each instance, while Utilization indicates the average CPU usage for computing the control inputs. Our algorithm maintains excellent safety performance on both the personal computer and the embedded platform. Notably, on the Jetson Nano@1.42GHz, the average computation time was as low as $9.6 \times 10^{-3}$s, highlighting the advantages of parallel computation and CUDA acceleration for neural networks. The test results also showed that our method occupies no more than 26\% of the hardware resources on the Jetson Nano, with parallel computing effectively allocating resources to maximize performance, sufficiently meeting the computational demands of the quadrotor flight control system, demonstrating that BAN-MPC can achieve safe and fast control on resource-constrained embedded platforms.

\subsection{Parameter-Adaptive Validation}

As shown in Fig. \ref{fg:Figure_8}, we apply different deviations $\Delta \theta$ to the model parameters (unicycle mass, unicycle tire friction, quadrotor mass, and quadrotor arm length) to validate the parameter adaptation capability of BAN-MPC. Results indicate that, with parameter deviations $\le15\%$ from nominal, safety is preserved and control error remains below $5\%$, well within engineering limits.

\begin{figure}[ht]
\centering
    \includegraphics[width=0.72\linewidth]{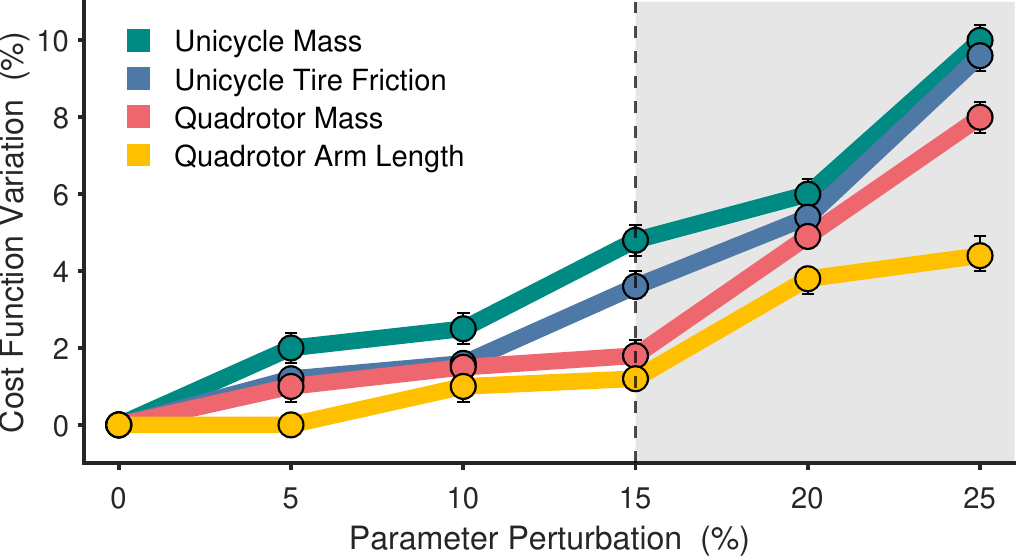}
    \vspace{-2mm}
    \caption{BAN-MPC's parameter adaptive performance validation} 
    \vspace{-2mm}
    \label{fg:Figure_8}
  
\end{figure}

\section{CONCLUSIONS}

We propose a novel MPC framework that employs Control Barrier Functions for obstacle avoidance to ensure safety in complex environments, where embedding a neural network-approximated value function into Short-horizon MPC objectives enables rapid online computation while neural sensitivities facilitate adaptive value function adjustments without retraining during parameter variations. Hardware-in-the-loop experiments on Jetson Nano demonstrate BAN-MPC's superiority over traditional MPC in control precision, computation speed, safety guarantees, and parameter adaptability --- though with inherent limitations: significant approximation errors emerge under large parametric deviations ($>15\%$) from nominal values, and while our current validation covers unicycle and quadrotor systems with observable parameter changes, real-world applications often involve unobservable parameters requiring estimation and high-dimensional systems like humanoids remain unexplored. To address these challenges, future work will develop robust adaptation strategies for extreme parametric shifts, integrate real-time estimation modules (e.g., neural-enhanced Kalman filtering) for unobservable dynamics, and extend deployment to safe control in high-dimensional robotic systems.





\bibliographystyle{IEEEtran}
\bibliography{root}

\begin{thebibliography}{10}
\providecommand{\url}[1]{#1}
\csname url@samestyle\endcsname
\providecommand{\newblock}{\relax}
\providecommand{\bibinfo}[2]{#2}
\providecommand{\BIBentrySTDinterwordspacing}{\spaceskip=0pt\relax}
\providecommand{\BIBentryALTinterwordstretchfactor}{4}
\providecommand{\BIBentryALTinterwordspacing}{\spaceskip=\fontdimen2\font plus
\BIBentryALTinterwordstretchfactor\fontdimen3\font minus \fontdimen4\font\relax}
\providecommand{\BIBforeignlanguage}[2]{{%
\expandafter\ifx\csname l@#1\endcsname\relax
\typeout{** WARNING: IEEEtran.bst: No hyphenation pattern has been}%
\typeout{** loaded for the language `#1'. Using the pattern for}%
\typeout{** the default language instead.}%
\else
\language=\csname l@#1\endcsname
\fi
#2}}
\providecommand{\BIBdecl}{\relax}
\BIBdecl

\bibitem{c5}
P.~Ogren and N.~Leonard, ``A convergent dynamic window approach to obstacle avoidance,'' \emph{IEEE Trans. Rob.}, vol.~21, no.~2, pp. 188--195, 2005.

\bibitem{c7}
S.~Nasiriany, H.~Liu, and Y.~Zhu, ``Augmenting reinforcement learning with behavior primitives for diverse manipulation tasks,'' \emph{Proc. IEEE Int. Conf. Robot. Autom. (ICRA)}, pp. 7477--7484, 2022.

\bibitem{c9}
J.~Park and H.~J. Kim, ``Online trajectory planning for multiple quadrotors in dynamic environments using relative safe flight corridor,'' \emph{IEEE Robot. Autom. Lett.}, vol.~6, no.~2, pp. 659--666, 2021.

\bibitem{y1}
K.~Wang, C.~Mu, Z.~Ni, and D.~Liu, ``Safe reinforcement learning and adaptive optimal control with applications to obstacle avoidance problem,'' \emph{IEEE Trans. Autom. Sci. Eng.}, vol.~21, no.~3, pp. 4599--4612, 2024.

\bibitem{y2}
H.~Ma, J.~Chen, S.~Eben, Z.~Lin, Y.~Guan, Y.~Ren, and S.~Zheng, ``Model-based constrained reinforcement learning using generalized control barrier function,'' \emph{Proc. IEEE/RSJ Int. Conf. Intell. Robots Syst.}, pp. 4552--4559, 2021.

\bibitem{y3}
N.~Schmid, J.~Gruner, H.~S. Abbas, and P.~Rostalski, ``A real-time gp based mpc for quadcopters with unknown disturbances,'' \emph{Proc. Am. Control Conf. (ACC)}, pp. 2051--2056, 2022.

\bibitem{y4}
T.~Salzmann, E.~Kaufmann, J.~Arrizabalaga, M.~Pavone, D.~Scaramuzza, and M.~Ryll, ``Real-time neural mpc: Deep learning model predictive control for quadrotors and agile robotic platforms,'' \emph{IEEE Robot. Autom. Lett.}, vol.~8, no.~4, pp. 2397--2404, 2023.

\bibitem{c10}
J.~Liu, M.~Li, J.~W. Grizzle, and J.-K. Huang, ``Clf-cbf constraints for real-time avoidance of multiple obstacles in bipedal locomotion and navigation,'' \emph{Proc. IEEE/RSJ Int. Conf. Intell. Robots Syst. (IROS)}, pp. 10\,497--10\,504, 2023.

\bibitem{a7}
A.~D. Ames, X.~Xu, J.~W. Grizzle, and P.~Tabuada, ``Control barrier function based quadratic programs for safety critical systems,'' \emph{IEEE Trans. Autom. Control}, vol.~62, no.~8, pp. 3861--3876, 2017.

\bibitem{c13}
T.~Jin, J.~Di, X.~Wang, and H.~Ji, ``Safety barrier certificates for path integral control: Safety-critical control of quadrotors,'' \emph{IEEE Robot. Autom. Lett.}, vol.~8, no.~9, pp. 6006--6012, 2023.

\bibitem{c14}
R.~Grandia, A.~J. Taylor, A.~D. Ames, and M.~Hutter, ``Multi-layered safety for legged robots via control barrier functions and model predictive control,'' \emph{Proc. IEEE Int. Conf. Robot. Autom.}, pp. 8352--8358, 2021.

\bibitem{c19}
X.~Zhang, M.~Bujarbaruah, and F.~Borrelli, ``Near-optimal rapid mpc using neural networks: A primal-dual policy learning framework,'' \emph{IEEE Trans. Control Syst. Technol.}, vol.~29, no.~5, pp. 2102--2114, 2021.

\bibitem{a4}
\BIBentryALTinterwordspacing
C.~A. Orrico, B.~Yang, and D.~Krishnamoorthy, ``On building myopic mpc policies using supervised learning,'' 2024. [Online]. Available: \url{https://arxiv.org/abs/2401.12546}
\BIBentrySTDinterwordspacing

\bibitem{a10}
\BIBentryALTinterwordspacing
D.~Xu, R.~Aerts, P.~Karamanakos, and M.~Lazar, ``Constraints-informed neural-laguerre approximation of nonlinear mpc with application in power electronics,'' 2024. [Online]. Available: \url{https://arxiv.org/abs/2409.09436}
\BIBentrySTDinterwordspacing

\bibitem{c22}
S.~Chen, K.~Saulnier, N.~Atanasov, D.~D. Lee, V.~Kumar, G.~J. Pappas, and M.~Morari, ``Approximating explicit model predictive control using constrained neural networks,'' \emph{Proc. Am. Control Conf.}, pp. 1520--1527, 2018.

\bibitem{c23}
\BIBentryALTinterwordspacing
S.~Mamedov, R.~Reiter, S.~M.~B. Azad, R.~Viljoen, J.~Boedecker, M.~Diehl, and J.~Swevers, ``Safe imitation learning of nonlinear model predictive control for flexible robots,'' 2024. [Online]. Available: \url{https://arxiv.org/abs/2212.02941}
\BIBentrySTDinterwordspacing

\bibitem{c26}
H.~Wang, H.~Zhao, and B.~Li, ``Bridging multi-task learning and meta-learning: Towards efficient training and effective adaptation,'' \emph{Proc. Int. Conf. Mach. Learn.}, pp. 10\,991--11\,002, 2021.

\bibitem{c27}
L.~Huang, W.~Zhao, Y.~Liu, D.~Yang, A.~W.-C. Liew, and Y.~You, ``An evidential multi-target domain adaptation method based on weighted fusion for cross-domain pattern classification,'' \emph{IEEE Trans. Neural Netw. Learn. Syst.}, vol.~35, no.~10, pp. 14\,218--14\,232, 2024.

\bibitem{c28}
S.~W. Chen, T.~Wang, N.~Atanasov, V.~Kumar, and M.~Morari, ``Large scale model predictive control with neural networks and primal active sets,'' \emph{Automatica}, vol. 135, p. 109947, 2022.

\bibitem{c29}
M.~Bogdanovic, M.~Khadiv, and L.~Righetti, ``Model-free reinforcement learning for robust locomotion using demonstrations from trajectory optimization,'' \emph{Frontiers in Robotics and AI}, vol.~9, 2022.

\bibitem{c30}
A.~V. Fiacco, ``Sensitivity analysis for nonlinear programming using penalty methods,'' \emph{Math. Programm.}, vol.~10, pp. 287--311, 1976.

\bibitem{a1}
\BIBentryALTinterwordspacing
H.~Hose, A.~Gräfe, and S.~Trimpe, ``Parameter-adaptive approximate mpc: Tuning neural-network controllers without retraining,'' 2024. [Online]. Available: \url{https://arxiv.org/abs/2404.05835}
\BIBentrySTDinterwordspacing

\bibitem{x1}
X.~Zhang, A.~Liniger, and F.~Borrelli, ``Optimization-based collision avoidance,'' \emph{IEEE Trans. Control Syst. Technol.}, vol.~29, no.~3, pp. 972--983, 2021.

\bibitem{a3}
S.~Ross, G.~Gordon, and D.~Bagnell, ``A reduction of imitation learning and structured prediction to no-regret online learning,'' \emph{Proc. 14th Int. Conf. Artif. Intell. Statist. Workshop Conf. Proc.}, pp. 627--635, 2011.

\bibitem{a6}
J.~A. Andersson, J.~Gillis, G.~Horn, J.~B. Rawlings, and M.~Diehl, ``Casadi: a software framework for nonlinear optimization and optimal control,'' \emph{Math. Prog. Comp.}, vol.~11, pp. 1--36, 2019.

\bibitem{a5}
\BIBentryALTinterwordspacing
L.~Manda, S.~Chen, and M.~Fazlyab, ``Domain adaptive safety filters via deep operator learning,'' 2024. [Online]. Available: \url{https://arxiv.org/abs/2410.14528}
\BIBentrySTDinterwordspacing

\bibitem{a8}
G.~Torrente, E.~Kaufmann, P.~Föhn, and D.~Scaramuzza, ``Data-driven mpc for quadrotors,'' \emph{IEEE Robot. Autom. Lett.}, vol.~6, no.~2, pp. 3769--3776, 2021.

\bibitem{k9}
S.~Ross and D.~Bagnell, ``Efficient reductions for imitation learning,'' \emph{Proceedings of the Thirteenth International Conference on Artificial Intelligence and Statistics}, vol.~9, pp. 661--668, 13--15 May 2010.

\bibitem{kk1}
A.~V. Fiacco, ``Sensitivity analysis for nonlinear programming using penalty methods,'' \emph{Mathematical programming}, vol.~10, no.~1, pp. 287--311, 1976.

\end{thebibliography}

\appendix
\subsection{Performance Guarantees of the VF-DAGGER Algorithm}
Traditional behavioral cloning trains policies by minimizing a surrogate loss under the expert's state distribution: $\hat{\pi}_{\text{sup}} = \arg \min_{\pi \in \Pi} \mathbb{E}_{s \sim d_{\pi^*}} \left[ \ell(s, \pi) \right]$, where $\Pi$ is the class of policies the learner is considering and $\ell$ is the observed surrogate loss function. This approach fundamentally ignores the state distribution shift during policy deployment. As proven by Ross and Bagnell \cite{k9} : for behavioral cloning policy $\pi$ achieving expected loss $\epsilon$ on expert states $d_{\pi^*}$, its $T$-step cumulative cost satisfies.

\newtheorem{theorem}{Theorem}
\begin{theorem}\label{thm:1}[The performance bound of behavior cloning \cite{k9}]
Let $\mathbb{E}_{s \sim d_{\pi^*}}\bigl[\ell(s,\pi)\bigr]=\epsilon$, then $J(\pi)\le J(\pi^*)+T^{2}\epsilon$.
\end{theorem}

Theorem 1 reveals that behavioral cloning exhibits quadratically growing performance error with task horizon $T$, inevitably failing in long-horizon tasks despite excellent performance on expert demonstrations.

When behavioral cloning is used to approximate the CBF-MPC value function $V^*$, the learned network $V'$ induces a policy $\pi'$ through the CBF-MPC controller. Since $\pi'$ is entirely generated by $V'$ through MPC optimization, it satisfies the conditions of Theorem 1. Relative to the expert policy (CBF-MPC) $\pi^*$, we have the following guarantee.

\begin{lemma}\label{lem:1}
Let $\mathbb{E}_{s \sim d_{\pi^*}}\bigl[\ell(s,\pi')\bigr]=\epsilon$, then $J(\pi')\le J(\pi^*)+T^{2}\epsilon$.
\end{lemma}

Lemma 1 establishes that the performance gap between $\pi'$ (the policy from value function imitation) 
and the optimal CBF-MPC policy $\pi^*$ grows quadratically with task horizon $T$. 
Consequently, $\pi'$ suffers significant performance degradation in long-horizon tasks. 
Moreover, since $\epsilon$ is estimated under the expert state distribution $d_{\pi^*}$, 
its value may be substantially underestimated during actual deployment due to distribution shift.

We now establish a performance guarantee for the policy learned by the VF-DAGGER algorithm. Let $V^*(s)$ denote the expert value function, $\hat V_i(s)$
the value function learned at iteration $i$, $V_i(s)$ the mixed value function, and $\pi_i$ the CBF-MPC policy induced by $V_i(s)$; the expert CBF-MPC policy is $\pi^*$. Define the average classification error $\epsilon_i=\mathbb{E}_{s \sim d_{\pi_i}}\bigl[\ell(s,\pi_i)\bigr]$, and let $Q_{t}^{\pi^*}(s,\pi)$ be the $t$-step cost when one action from $\pi$ is taken at state $s$ and $\pi^*$ is followed for the remaining $t-1$ steps.

Assume there exists a constant $u$ such that, for every state $s$ and action $a$, $Q^{\pi^*}_{T-t+1}(s,a)\;-\;Q^{\pi^*}_{T-t+1}(s,\pi^*)\le  u $, meaning that the cost increase caused by any single-step action error is uniformly bounded by $u$. Let ${\pi}_{1:n_{D}}$ denote the sequence of policies ${\pi}_{1},{\pi}_{2},...,{\pi}_{n_{D}}$, where ${n_{D}}$ denotes the number of iterations. Let $\epsilon^* = \min_{\pi\in\Pi}\frac{1}{{n_{D}}}\sum_{i=1}^{n_{D}} \mathbb{E}_{s\sim d_{\pi_i}}\bigl[\ell(s,\pi)\bigr]$ be the true loss of the best policy in hindsight, where $\ell$ is the observed surrogate loss function. Then the following holds in the infinite sample case (infinite number of sample trajectories at each iteration).

\begin{theorem}\label{thm:2}
After running VF-DAGGER for $n_{D} =\tilde {\mathcal{O}}(uT)$ iterations, there exists a policy ${\pi} \in{\pi}_{1:{n_{D}}}$ such that $J(\pi)\le J(\pi^*)+uT\epsilon^* + \mathcal{O}(1)$.
\end{theorem}

\begin{proof}
Given the current iterate $\pi_i$, consider the policy $\pi^{(i)}_{1:t}$, which executes $\pi_i$ in the first $t$-steps and then execute the expert $\pi^*$. Hence $\pi^{(i)}_{1:T}$ denotes the policy that executes $\pi_i$ for all $T$ steps (i.e., $\pi^{(i)}_{1:T} = \pi_i$), and $\pi^{(i)}_{1:0}$ denotes the policy that executes the expert from the very beginning (i.e., $\pi^{(i)}_{1:0} = \pi^*$). Then,
\begin{subequations}
\begin{align}
&J(\pi_i)-J(\pi^*)=J(\pi^{(i)}_{1:T})-J(\pi^{(i)}_{1:0})\\
&= \sum_{t=1}^{T}\Bigl[J\bigl(\pi^{(i)}_{1:T-t+1}\bigr) - J\bigl(\pi^{(i)}_{1:T-t}\bigr)\Bigr] \quad\text{(telescoping sum)}\\
&= \sum_{t=1}^{T} \mathbb{E}_{s\sim d^t_{\pi_i}}\bigl[\,Q^{\pi^*}_{T-t+1}(s,\pi_i(s))\;-\;Q^{\pi^*}_{T-t+1}(s,\pi^*(s))\bigr] \\
&\le  u \sum_{t=1}^{T} \mathbb{E}_{s\sim d^t_{\pi_i}}\bigl[\ell(s,\pi_i)\bigr] \\
&=  u\,T\epsilon_i
\end{align}
\end{subequations}

Assume each iteration applies Follow-The-Leader: $\pi_{i}=\arg\!\min_{\pi\in\Pi}
\sum_{j=1}^{\,i-1}
\mathbb{E}_{s\sim d_{\pi_j}}\![\ell(s,\pi)]$, then the cumulative loss over ${n_{D}}$ iterations satisfies:
\begin{align}
&\sum_{i=1}^{n_{D}} \epsilon_i \le {n_{D}}\,\epsilon^* +\mathcal{O}\bigl(\ln {n_{D}}\bigr)\\
&\epsilon^* =\min_{\pi\in\Pi}\frac{1}{{n_{D}}}\sum_{i=1}^{n_{D}} \mathbb{E}_{s\sim d_{\pi_i}}\bigl[\ell(s,\pi)\bigr]
\end{align}

\newtheorem{remark}{Remark}
\begin{remark}[Choice of online learner]
We analyse VF-DAgger with FTL purely for convenience; any online learner with
$\mathcal{O}(\ln n_D)$ (strong-convex) regret would suffice.
We pick FTL because its logarithmic regret bound makes the
$uT\epsilon^* + \mathcal{O}(1)$ performance guarantee transparent,
even though VF-DAgger itself can use any batch supervised learner.
\end{remark}

Consequently, $\frac{1}{{n_{D}}}\sum_{i=1}^{n_{D}} \epsilon_i \;\le\;\epsilon^* + \tfrac{\mathcal{O}\!\bigl(\ln {n_{D}}\bigr)}{{n_{D}}}$, so there exists at least one iteration index $\hat i$ for which $\epsilon_{\hat
 i}\;\le\;\epsilon^* + \tfrac{\mathcal{O}\!\bigl(\ln {n_{D}}\bigr)}{{n_{D}}}$. Setting $n_{D} = \tilde{\mathcal{O}}(uT)$ (i.e., $n_{D}$ is proportional to $uT$, up to logarithmic factors), we obtain $\frac{\mathcal{O}\!\bigl(\ln n_{D}\bigr)}{n_{D}}
 \;=\;
\frac{\mathcal{O}\!\bigl(\ln(uT)\bigr)}{\tilde{\mathcal{O}}(uT)}
 \;=\;
\tilde{\mathcal{O}}\!\bigl(\tfrac{1}{T}\bigr)$. Choose the $\hat{i}$-th iterate ${\pi} = \pi_{\hat{i}}$; combining this selection with the single-iteration performance bound yields the desired result.
\begin{subequations}
\begin{align}
J({\pi})
&\le J(\pi^*) + u T \epsilon_{\hat
 i}\\
&\le J(\pi^*) + u T\bigl(\epsilon^* + \tilde{\mathcal{O}}\bigl(\tfrac1T\bigr) \bigr)\\
&= J(\pi^*) + u T \epsilon^* + uT\cdot \tilde{\mathcal{O}}\bigl(\tfrac1T\bigr) \\
&= J(\pi^*) + u T \epsilon^* + \mathcal{O}(1)
\end{align}    
\end{subequations}
\end{proof}

Theorem 2 demonstrates that after learning the neural value function via VF-DAGGER with ${n_{D}} = \tilde {\mathcal{O}}(uT)$ iterations, the performance error between the constructed policy and the original CBF-MPC policy scales sublinearly with task horizon $T$. This effectively resolves the quadratic growth issue of behavioral cloning in Lemma 1.

In the finite-sample regime, suppose that at each iteration $i$ we collect $l=\mathcal{O}(1)$ trajectories, forming the dataset $\mathcal{D}_i$. Let $\epsilon' \;=\; \min_{\pi\in\Pi}\;\frac{1}{{n_{D}}}\sum_{i=1}^{n_{D}} \mathbb{E}_{s\sim {\mathcal{D}_i}}\bigl[\ell(s,\pi)\bigr]$ denote the training loss of the best policy on the sampled trajectories. Then the following holds in the finite sample case.
\begin{theorem}\label{thm:3}
After ${n_{D}} = \mathcal{O}\bigl(u^2T^2 \,\log\tfrac{1}{\delta}\bigr)$ VF-DAGGER iterations, collecting $l=\mathcal{O}(1)$ trajectories at each round, there exists a policy ${\pi} \;\in\;{ \pi}_{1:{n_{D}}}$ such that, with probability at least $1-\delta$, $J( \pi)\le J(\pi^*)+uT\epsilon' + \mathcal{O}(1)$.
\end{theorem}

\begin{proof}
Define per-iteration losses:
\begin{align}
\epsilon_i =
  \mathbb E_{s\sim d_{\pi_i}}\!\bigl[\ell(s,\pi_i)\bigr],
  \hat\epsilon_i =
\frac1{|\mathcal{D}_i|}\sum_{s\in\mathcal{D}_i}\ell(s,\pi_i)
\end{align}

With $\ell$ bounded in $[0,1]$, the difference $Z_i=\epsilon_i-\hat\epsilon_i$ satisfies $|Z_i|\le 1$. The martingale
$S_{n_D}=\sum_{i=1}^{n_D}Z_i$ obeys by Azuma-Hoeffding:
  $\Pr\!\Bigl(
        |S_{n_D}|
        >
        \sqrt{\tfrac12\,{n_D}\ln\tfrac2\delta}
     \Bigr)
  \;\le\;\delta
$.

Thus with probability $1-\delta$:
\begin{align}
   \sum_{i=1}^{n_D}\epsilon_i
   \;\le\;
   \sum_{i=1}^{n_D}\hat\epsilon_i
   +\sqrt{\tfrac12\,{n_D}\ln\tfrac2\delta}
\end{align}

Using Follow-The-Leader (or any no-regret learner with
$\mathcal{O}(\ln n_D)$ regret for strongly convex losses) on the
empirical losses,
\begin{align}
  \sum_{i=1}^{n_D}\hat\epsilon_i
  \;\le\;
  n_D\,\epsilon' + \mathcal{O}(\ln n_D)
\end{align}

Thus with probability $1-\delta$,
\begin{align}
\bar{\epsilon}\;:=\;
\frac{1}{n_D}\sum_{i=1}^{n_D}\epsilon_i
\;\le\;
\epsilon'
\;+\;
\mathcal{O}\!\Bigl(\tfrac{\ln n_D}{n_D}\Bigr)
\;+\;
\mathcal{O}\!\Bigl(\sqrt{\tfrac{\ln(1/\delta)}{n_D}}\Bigr)
\end{align}

By the mean-value principle, there exists at least one iteration index $\hat i$ for which $\epsilon_{\hat{i}} \;\le\; \bar{\epsilon}$. Using the single-step bound $J(\pi_i)-J(\pi^*)\le uT\,\epsilon_i$ and selecting ${\pi}=\pi_{\hat i}$ yields:
\begin{align}
J({\pi}) \;\le\; J(\pi^*) 
\;+\; uT\!\left(
      \epsilon' 
      \;+\; \mathcal{O}\!\Bigl(\tfrac{\ln n_D}{n_D}\Bigr)
      \;+\; \mathcal{O}\!\Bigl(\sqrt{\tfrac{\ln(1/\delta)}{n_D}}\Bigr)
\right)
\end{align}

Choose $n_D=C\,u^{2}T^{2}\log(1/\delta)$
for a suitable constant $C>0$. Then
$
\frac{\ln n_D}{n_D} = \mathcal{O}\!\Bigl(\tfrac{1}{T^{2}}\Bigr), 
\quad
\sqrt{\frac{\ln(1/\delta)}{n_D}} = \mathcal{O}\!\Bigl(\tfrac{1}{T}\Bigr)
$. Multiplying by $uT$ yields $
uT \,\cdot\, \mathcal{O}\!\Bigl(\tfrac{1}{T^{2}}\Bigr)
   =
   \mathcal{O}\!\Bigl(\tfrac{1}{T}\Bigr),
\quad
uT \,\cdot\, \mathcal{O}\!\Bigl(\tfrac{1}{T}\Bigr)
   =
   \mathcal{O}(1)$. Substituting this into the previous inequality completes the proof.
\end{proof}

\begin{remark}[Trajectory count $l$]
When $l=\mathcal{O}(1)$, $l$-dependent terms are absorbed into $\mathcal{O}(1)$. Retaining the $l$-dependence explicitly shows that the required iterations scale as $n_{D}=\mathcal{O}\!\left(\frac{u^{2}T^{2}}{l}\ln\frac{1}{\delta}\right)$, while the additive error decays as $\mathcal{O}\!\left(l^{-1/2}\right)$. Hence, increasing $l$ both accelerates training and tightens the guarantees.
\end{remark}

A comparison of \hyperref[lem:1]{Lemma 2}, \hyperref[thm:2]{Theorem 2}, and \hyperref[thm:3]{Theorem 3} clearly highlights VF-DAGGER's advantage in addressing distribution shift.

\subsection{Probabilistic Practical Exponential Stability}

\begin{assumption} There exist constants $\alpha_1, \alpha_2, c > 0$ such that: 
\begin{align}
\alpha_1 \|x\|^2 \le V_{\text{MPC}}(x, \theta_{\text{nom}}) \le \alpha_2 \|x\|^2\\
\Delta V_{\text{MPC}}(x, \theta_{\text{nom}}) \le -c \|x\|^2
\end{align}
\end{assumption}

\begin{assumption}
\begin{align}
\Pr\left( \|e_V(x)\| \le \varepsilon_V \right) \ge 1 - \delta\\
\Pr\left( \|e_{\nabla}(x)\| \le \varepsilon_{\nabla} \right) \ge 1 - \delta
\end{align}
where $e_V(x):=V_{\text{NN}}(x)-V_{\text{MPC}}(x,\theta_{\text{nom}})$, $e_{\nabla}(x):=\nabla_{V_{\text{NN}}}(x)-\left.\frac{\partial}{\partial\theta}V_{\text{MPC}}(x,\theta)\right|_{\theta_{\text{nom}}}$.
\end{assumption}

\begin{assumption} There exists $\gamma>0$ such that: 
\begin{align}
\|\phi(x_{k+1}, \theta) - \phi(x_k, \theta)\|
&\le \gamma\,\bigl\|V_{\text{BAN-MPC}}(x_k, \theta)\notag\\
&\qquad\quad - V_{\text{MPC}}(x_k, \theta)\bigr\|
\end{align}
where $\phi(x, \theta) = V_{\text{BAN-MPC}}(x, \theta) - V_{\text{MPC}}(x, \theta)$.
\end{assumption}

\begin{assumption} There exists $M_\theta > 0$ such that the parameter deviation from the nominal value satisfies: $\|\theta - \theta_{\text{nom}}\| \leq M_\theta$, and within this bound, the active constraint set remains unchanged.
\end{assumption}

\begin{theorem}[Probabilistic Practical Exponential Stability]\label{thm:4} 
If Assumptions 3-6 hold, when training with the VF-DAGGER algorithm (with ${n_D} = \mathcal{O}\left( u^2 T^2 \ln \frac{1}{\delta} \right)$ epochs and $l=\mathcal{O}(1)$ trajectories per epoch), there exists $\gamma \in \left(0, \frac{c}{\alpha_2} \right)$ such that, with probability at least $(1 - \delta)^2$, the closed-loop system enjoys probabilistic practical exponential stability:
$
\|x_k\| \le \kappa \|x_0\| e^{-\lambda k} + r,\forall\, k \ge 0
$, where the convergence rate is given by $\lambda = \frac{\gamma \alpha_2}{c - \gamma \alpha_2} > 0$, the radius of attraction is $r = \sqrt{\frac{\Gamma_V + \Gamma_\theta}{c - \gamma \alpha_2}}$, with $\Gamma_V = \gamma \varepsilon_V$ denoting the value function approximation error term and $\Gamma_\theta = \left[ 2L + \gamma (\varepsilon_{\nabla} + L_V) \right] \|\theta - \theta_{\text{nom}}\|$ representing the parameter sensitivity term.
\end{theorem}

\begin{proof}

\textbf{1) Neural Network Approximation Error Convergence}

By \hyperref[thm:2]{Theorem 2} and \hyperref[thm:3]{Theorem 3} of VF-DAGGER, there exists a policy $\pi$ such that the single-step expected errors satisfy $\|e_V(x)\| \le u T \varepsilon = \mathcal{O}(1)$, and $\|e_{\nabla}(x)\| \le u_{\nabla} T \varepsilon = \mathcal{O}(1)$ with probability $1 - \delta$. Increasing the training epochs $n_D$ and trajectories per epoch $l$ yields the error decay $\varepsilon_V = \mathcal{O}\left( \frac{1}{T} \right)$ and $\varepsilon_{\nabla} = \mathcal{O}\left( \frac{1}{T} \right)$.

\textbf{2) Value Function Error Bound}

From \cite{kk1} on MPC sensitivity, there exists a constant $L_V>0$ such that:
\begin{align}
\|V_{\text{MPC}}(x,\theta_{\text{nom}})+\left.\frac{\partial}{\partial\theta}V_{\text{MPC}}(x,\theta)\right|_{\theta_{\text{nom}}}(\theta-\theta_{\text{nom}})\notag\\
-V_{\text{MPC}}(x,\theta)\|\leq L_V||\theta-\theta_{\text{nom}}||
\end{align}

Within the event $E = \left\{\, |e_V| \le \varepsilon_V \;\cap\; \|e_{\nabla}\| \le \varepsilon_{\nabla} \,\right\}$, the value function error between BAN-MPC and the original CBF-MPC can be expressed as:
\begin{subequations}
\begin{align}
\eta(x, \theta)
&=\|V_{\text{BAN-MPC}}(x, \theta) - V_{\text{MPC}}(x, \theta)\|\\
&= \|V_{\text{NN}}(x) + \nabla_{V_{\text{NN}}}(x)(\theta - \theta_{\text{nom}}) - V_{\text{MPC}}(x, \theta)\| \\
&= \|e_{V}(x) + V_{\text{MPC}}(x, \theta_{\text{nom}}) + ( \frac{\partial}{\partial \theta} V_{\text{MPC}}(x, \theta) \bigg|_{\theta_{\text{nom}}} \nonumber\\
&\qquad + e_{\nabla}(x) )(\theta - \theta_{\text{nom}}) - V_{\text{MPC}}(x, \theta)\| \\
&\leq \|e_{V}(x)\| + \|e_{\nabla}(x)(\theta - \theta_{\text{nom}})\| 
+ \|V_{\text{MPC}}(x, \theta_{\text{nom}})\nonumber \\
&\qquad+ \frac{\partial}{\partial \theta} V_{\text{MPC}}(x, \theta) \bigg|_{\theta_{\text{nom}}} (\theta - \theta_{\text{nom}}) - V_{\text{MPC}}(x, \theta)\| \\
&\leq \|e_{V}(x)\| + \|e_{\nabla}(x)(\theta - \theta_{\text{nom}})\| + L_V\|\theta - \theta_{\text{nom}}\|\\
&\leq \varepsilon_V+(\varepsilon_{\nabla} +L_V)||\theta-\theta_{\text{nom}}||
\end{align}
\end{subequations}
This event occurs with probability $\Pr(E) \ge (1 - \delta)^2$, and all subsequent derivations are conditioned on $E$.

\textbf{3) Lyapunov Difference Decomposition}

Define error function $\phi(x, \theta) = V_{\text{BAN-MPC}}(x, \theta) - V_{\text{MPC}}(x, \theta)$, $\Delta V_{\text{BAN-MPC}}(x_k)= V_{\text{BAN-MPC}}(x_{k+1}) - V_{\text{BAN-MPC}}(x_k)$, and $\Delta V_{\text{MPC}}(x_k)= V_{\text{MPC}}(x_{k+1}) - V_{\text{MPC}}(x_k)$. Then:
\begin{subequations}
\begin{align}
&\Delta V_{\text{BAN-MPC}}(x_k, \theta)= V_{\text{BAN-MPC}}(x_{k+1}, \theta) - V_{\text{BAN-MPC}}(x_k, \theta)\\
&=\Delta V_{\text{MPC}}(x_k, \theta) + \| \phi(x_{k+1}, \theta) - \phi(x_k, \theta) \|\\
&\le \Delta V_{\text{MPC}}(x_k,\theta) + \gamma\,\|V_{\text{BAN-MPC}}(x_k, \theta) - V_{\text{MPC}}(x_k, \theta)\|\\
&= \underbrace{\Delta V_{\text{MPC}}(x_k,\theta)}_{\text{Term A}}+\underbrace{\gamma\, \eta(x_k,\theta)}_{\text{Term B}}
\end{align}
\end{subequations}

From \cite{kk1}, there exists $L$ such that: $
\|V_{\text{MPC}}(x,\theta)-V_{\text{MPC}}(x,\theta_{\text{nom}})\|
\leq L||\theta-\theta_{\text{nom}}||$. Combining with nominal stability:
\begin{subequations}
\begin{align}
&\Delta V_{\text{MPC}}(x_k, \theta)= V_{\text{MPC}}(x_{k+1}, \theta) - V_{\text{MPC}}(x_k, \theta) \\
&=
\Big[
V_{\text{MPC}}(x_{k+1}, \theta_{\text{nom}})
- V_{\text{MPC}}(x_k, \theta_{\text{nom}})
\Big] \nonumber\\
&\qquad+\Big[
V_{\text{MPC}}(x_{k+1}, \theta)- V_{\text{MPC}}(x_{k+1}, \theta_{\text{nom}})
\Big] \nonumber\\
&\qquad-\Big[
V_{\text{MPC}}(x_k, \theta)
- V_{\text{MPC}}(x_k, \theta_{\text{nom}})
\Big]\\
&\le -c \|x_k\|^2 + 2 L \|\theta - \theta_{\text{nom}}\|
\end{align}
\end{subequations}
Then:
\begin{subequations}
\begin{align}
\Delta V_{\text{BAN-MPC}} 
&\le-c \|x_k\|^2
+ 2L \|\theta - \theta_{\text{nom}}\|\nonumber
\\
&\qquad+ \gamma \left[ \varepsilon_V + (\varepsilon_{\nabla} + L_V)\|\theta - \theta_{\text{nom}}\| \right]\\
&\le -c \|x_k\|^2+\gamma \varepsilon_V\nonumber\\
&\qquad+ 
\left[ 2L + \gamma(\varepsilon_{\nabla} + L_V)\right]\|\theta - \theta_{\text{nom}}\| \\
&= -c \|x_k\|^2 + \Gamma_V + \Gamma_\theta
\end{align}
\end{subequations}
where $\Gamma_V = \gamma \varepsilon_V$, $\Gamma_\theta = \left[ 2L + \gamma (\varepsilon_{\nabla} + L_V) \right] \|\theta - \theta_{\text{nom}}\|$.

\textbf{4) Stability Condition Construction}

Define attraction radius $r = \sqrt{\frac{\Gamma_V + \Gamma_\theta}{c - \gamma \alpha_2}}$ ($
\gamma \in \left(0, \frac{c}{\alpha_2} \right)$). When $\|x_k\| > r$:
\begin{subequations}
\begin{align}
\Delta V_{\text{BAN-MPC}}
&\le -c \|x_k\|^2 + \Gamma_V + \Gamma_\theta \\
&< -c r^2 + \Gamma_V + \Gamma_\theta \\
&= -c \left( \frac{\Gamma_V + \Gamma_\theta}{c - \gamma \alpha_2} \right) + \Gamma_V + \Gamma_\theta \\
&= (\Gamma_V + \Gamma_\theta) \left( \frac{-\gamma \alpha_2}{c - \gamma \alpha_2} \right)\\
&< 0
\end{align}
\end{subequations}

\textbf{5) Exponential Decay Construction}

The value function satisfies two-sided inequalities: $\alpha_1 \|x\|^2 - D \;\le\; V_{\text{BAN-MPC}}(x, \theta) \;\le\; \alpha_2 \|x\|^2 + D$, where $D = \varepsilon_V + (\varepsilon_{\nabla} + L_V) M_\theta + L M_\theta$ is the total error bound. The Lyapunov function evolution obeys:
$
V_{\text{BAN-MPC}}(x_{k+1}, \theta)
\le V_{\text{BAN-MPC}}(x_k, \theta) - \lambda (\Gamma_V + \Gamma_\theta)$, 
with $\lambda = \frac{\gamma \alpha_2}{c - \gamma \alpha_2} > 0$, and further: $
\Delta V_{\text{BAN-MPC}} \le -\lambda \alpha_1 \|x_k\|^2 + \lambda D$. Establishing the recurrence relation: $
V_{\text{BAN-MPC}}(x_{k+1}, \theta)
\le \left( 1 - \frac{\lambda \alpha_1}{\alpha_2} \right) V_{\text{BAN-MPC}}(x_k, \theta)
+ \lambda D$, where $
\beta = 1 - \frac{\lambda \alpha_1}{\alpha_2} \in (0, 1)$. The recursive solution is:
\begin{subequations}
\begin{align}
V_{\text{BAN-MPC}}(x_k, \theta) 
&\le \beta^k V_{\text{BAN-MPC}}(x_0, \theta)
+ \lambda D \sum_{i=0}^{k-1} \beta^i\\
&\le \beta^k V_{\text{BAN-MPC}}(x_0, \theta)
+ \frac{\lambda D }{1 - \beta}\\
&= \beta^k V_{\text{BAN-MPC}}(x_0, \theta)
+ \frac{ D \alpha_2}{\alpha_1}\\
\end{align}
\end{subequations}

\textbf{6) State Norm Bound}

Combining the Lyapunov bounds, the state norm satisfies:
\begin{align}
\alpha_1 \|x_k\|^2 - D &\le \beta^k (\alpha_2 \|x_0\|^2 + D) + \frac{D \alpha_2}{\alpha_1} \nonumber\\
\Rightarrow \quad
\|x_k\|^2 &\le
\underbrace{
\frac{\alpha_2}{\alpha_1} \beta^k \|x_0\|^2
}_{\text{exponential term}}
+
\underbrace{
\frac{D}{\alpha_1} \beta^k +
\frac{D \alpha_2}{\alpha_1^2} +
\frac{D}{\alpha_1}
}_{\text{bias term}}
\end{align}
Let $\kappa = \sqrt{\frac{\alpha_2}{\alpha_1}}$, $\beta^k = \left(1 - \frac{\lambda \alpha_1}{\alpha_2}\right)^k \le e^{-\lambda k}$. As $T \rightarrow \infty$, the bias term is $\mathcal{O}(r^2)$, yielding rigorously: 
\begin{align}
\|x_k\| \le \kappa \|x_0\| e^{-\lambda k} + r
\end{align}
The above holds within event $E$ with $\Pr(E) \ge (1 - \delta)^2$, proving practical exponential stability with probability at least $
(1 - \delta)^2$.

\end{proof}

\begin{figure*}[ht]
\centering
    \includegraphics[width=\linewidth]{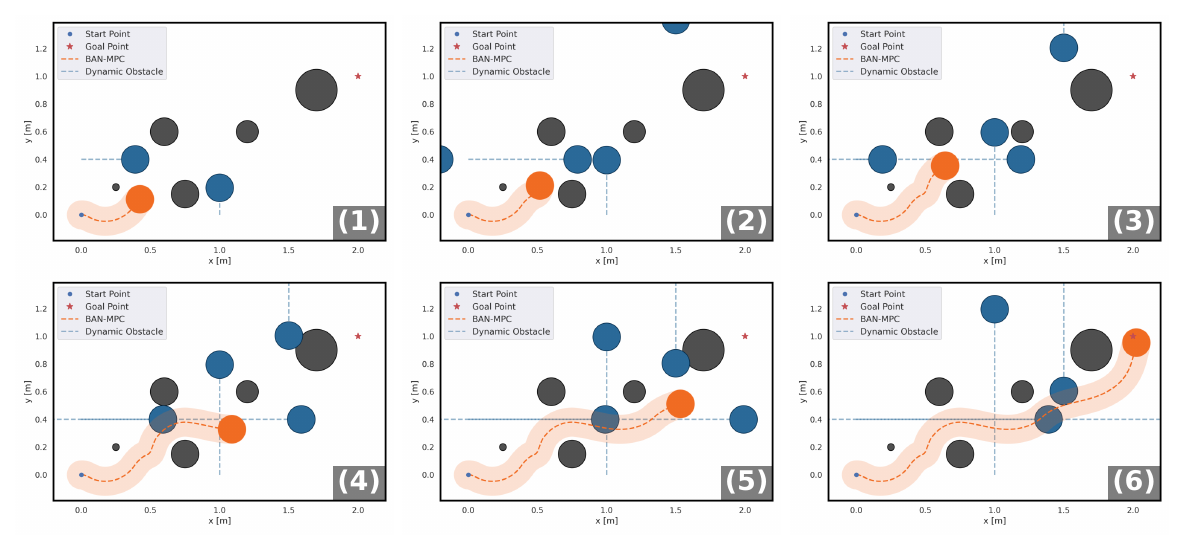}
    \vspace{-2mm}
    \caption{Dynamic obstacle-avoidance task using BAN-MPC} 
    \vspace{-2mm}
    \label{fg:Figure_9}
\end{figure*}

\subsection{Dynamic Obstacle-Avoidance Task}

To demonstrate the framework's scalability, we extend the unicycle navigation task by introducing 4 circular dynamic obstacles, validating BAN-MPC's obstacle avoidance capability in hybrid static-dynamic scenarios. The implementation comprises three key steps: 
\begin{itemize}

\item Online integration of a trajectory prediction module to estimate time-varying obstacle centers $o(t)$;

\item Reconstruction of safety functions into time-varying formulations $H^{(m)}(x,t)$; 

\item Explicit incorporation of temporal dimensions into CBF constraints $\Delta H({x}_k,{u}_k,t_k)+\gamma H({x}_k,t_k)\geq0$. 
\end{itemize}

Experimental results confirm that BAN-MPC with dynamic CBFs successfully avoids both static and dynamic obstacles while reaching the target, as depicted in Fig. \ref{fg:Figure_9}.

\end{document}